\documentclass[twoside]{article}
\usepackage[preprint]{aistats2027}
\usepackage[T1]{fontenc}
\usepackage[utf8]{inputenc}
\usepackage{amsmath,amssymb,amsthm,mathtools,booktabs,longtable,graphicx}
\usepackage[round]{natbib}
\usepackage[hidelinks]{hyperref}
\usepackage{balance}

\newtheorem{theorem}{Theorem}
\newtheorem{proposition}{Proposition}
\theoremstyle{definition}
\newtheorem{assumption}{Assumption}
\newcommand{\E}{\mathbb E}
\newcommand{\R}{\mathbb R}
\newcommand{\norm}[1]{\left\lVert#1\right\rVert}
\newcommand{\eps}{\varepsilon}
\begin{document}
\runningauthor{A. Fern\'andez-Hern\'andez, C. P\'erez-Corral, J. I. Mestre, M. F. Dolz, E. S. Quintana-Ort\'i}
\twocolumn[
\aistatstitle{Early Memory Selection for Balanced Adam}
\aistatsauthor{
Alberto Fern\'andez-Hern\'andez
\And
Cristian P\'erez-Corral
\And
Jose I. Mestre
}
\aistatsaddress{
Universitat Polit\`ecnica de Val\`encia\\
Valencia, Spain\\
\texttt{a.fernandez@upv.es}
\And
Universitat Polit\`ecnica de Val\`encia\\
Valencia, Spain\\
\texttt{cpercor@upv.es}
\And
Universitat Polit\`ecnica de Val\`encia\\
Valencia, Spain\\
\texttt{jimesmir@upv.es}
}
\aistatsauthor{Manuel F. Dolz\And Enrique S. Quintana-Ort\'i}
\aistatsaddress{Universitat Jaume I\\ Castell\'o de la Plana, Spain\\ \texttt{dolzm@uji.es} \And Universitat Polit\`ecnica de Val\`encia\\ Valencia, Spain\\ \texttt{quintana@disca.upv.es}}
]
\begin{abstract}
We propose a method for choosing the shared memory parameter $\beta_1=\beta_2=\beta$ in Adam from a short pilot training. The selected $\beta$ remains fixed during the subsequent full training. A local model of Adam's normalized direction balances sampling variability against the delay introduced by averaging past gradients. This balance gives a cubic memory rule, whose two coefficients are estimated from gradient probes at a few pilot checkpoints. The estimator uses the numerator and denominator jointly, preserving their covariance. With a 200-update pilot and sixteen probe gradients at each of four checkpoints, a seed-matched retrospective evaluation on eleven vision and language workloads reduces mean relative validation gap by 40.7\% and worst-quarter mean gap by 44.3\% against the grid representative of shared $\beta$ 0.95. The mean gap is also 32.3\% lower than that of the best constant $\beta$ chosen across all eleven workloads.
\end{abstract}

\section{Introduction}
Adam's memory is a statistical timescale: it determines which past gradients still influence the current update \citep{kingma2015adam}. A short memory responds quickly to changes and retains more sampling variability. A long memory averages that variability more effectively and follows a changing gradient with a delay. The useful timescale therefore depends on the dynamics of the task. A single fixed default makes a compromise across those dynamics; early measurement offers a way to choose the timescale for the workload at hand.

Recent language-model experiments show that setting the two parameters equal can retain strong training performance \citep{orvieto2025secretsauce}. This family is called \emph{balanced Adam}: $\beta_1=\beta_2=\beta$. It reduces tuning to one memory while preserving coordinatewise normalization. A shared value around 0.95 is a useful empirical reference. The remaining choice still matters: one memory per training is different from one memory for every training. This raises a concrete question: \emph{can a short observation of gradient dynamics choose a useful fixed memory for each workload?}

Adam normalizes an averaged gradient by an averaged squared gradient, so the errors of its two moments act jointly. We retain their covariance and analyze how averaging reduces fluctuations while delaying the response to a changing direction. A first-order local model separates these two effects and gives a memory rule in units $H=(1-\beta)^{-1}$. Grouped gradient probes and checkpoint differences estimate its coefficients. Completed fixed-memory trainings evaluate the resulting choices.

The contributions are:
\begin{itemize}
\item A joint-moment tracking objective that separates retained noise from delay and yields a cubic memory law. Its sensitivity formula explains why useful calibration tolerates coefficient error.
\item An explicit early finite-probe rule and a finite-resolution lag interpretation, with a quantified population-model allowance for endpoint uncertainty.
\item An eleven-workload evaluation against every fixed grid memory, with matched training recipes, task-level calibration, dense timing sensitivity, and gradient-equivalent cost accounting.
\end{itemize}
Different fixed $\beta$ values serve different tasks well. Our selector uses gradients to choose a memory for each run. On the evaluated suite it reduces mean, maximum, and tail gap compared with every fixed grid value, including the best constant selected from all validation outcomes.
Code and reproducible evaluation material are available at \url{https://github.com/AlbertoFdezHdez/Adam_beta_rule_cubic}.

\section{Related work}
We distinguish analyses of Adam's moments, theoretical relationships between its $\beta$ values, and methods that learn optimizer parameters during training. Our question is the early choice of one fixed shared memory.
\paragraph{Adam's statistical structure and shared memories.}
\citet{balles2018dissecting} analyze the interaction between sign, magnitude, and variance in Adam. \citet{orvieto2025secretsauce} provide extensive empirical support for equal memory parameters and a statistical perspective on Adam. These results motivate the shared-memory family. A complementary budget-based rule relates its memory to a useful training horizon \citep{fernandez2026refresh}. The present rule obtains its selection signal from early normalized-direction variability and temporal change.

\paragraph{\texorpdfstring{$\beta$}{β} relationships and optimization guarantees.}
\citet{ahn2024ftrl} interpret Adam through online learning of updates and follow-the-regularized-leader. \citet{nguyen2026beta} analyze both sides of $\beta_1=\sqrt{\beta_2}$, whose optimality depends on the adversary model. Their criterion relates the memories under online regret; ours chooses a shared memory from an observed gradient process under local tracking risk. Stochastic-differential-equation analyses explain scaling with batch size \citep{malladi2022sdes}, whereas we hold the training recipe fixed and measure its early dynamics.

\paragraph{Learning optimizer parameters.}
\citet{chandra2022gradient} differentiate through optimizers to learn their hyperparameters, including both Adam memories. MADA jointly learns memory and optimizer-interpolation coefficients through online hypergradient descent; its fixed-state variant reuses the optimizer learned at the end of another training \citep{ozkara2024mada}. Our rule instead uses an early noise--lag measurement to choose one shared $\beta$, fixed for full training.

\paragraph{Adaptive momentum and moment estimation.}
\citet{topollai2026adaptive} derive online momentum adaptation from a two-plane objective model. Their AdamW variant adapts the first-moment memory while keeping the second-moment memory fixed, with layerwise adaptation in language models. YellowFin estimates curvature and gradient statistics to tune SGD's momentum and learning rate \citep{zhang2019yellowfin}. Kalman-Adam estimates first and second moments using separate Bayesian filters \citep{ali2026kalman}. The present method retains Adam's exponential averages and calibrates their common timescale early, using fluctuations and change of their joint normalized target. Online memory adaptation and early fixed shared-memory selection provide complementary ways to use information from the training problem.

\section{What should the memory track?}\label{sec:target}
Choosing a memory starts with identifying what Adam's numerator and denominator jointly estimate. This statistical reference gives a directly analyzable objective: the tracking error of the normalized update direction.
\subsection{Adam and its population reference}
Let $\theta_t\in\R^d$ be the parameters and $g_t\in\R^d$ an effective-batch gradient, including accumulation and the recipe's gradient processing. From zero initial moments Adam computes
\begin{align}
 m_t&=\beta_1m_{t-1}+(1-\beta_1)g_t,\nonumber\\
 v_t&=\beta_2v_{t-1}+(1-\beta_2)g_t^{\odot2},\label{eq:adam}\\
 u_t&=\frac{m_t/(1-\beta_1^t)}{\sqrt{v_t/(1-\beta_2^t)}+\eps}.\nonumber
\end{align}
The superscript $\odot2$ means coordinatewise squaring. Square roots and division are coordinatewise, with $\eps>0$ outside the root. The gradient update is $-\eta_tu_t$; AdamW additionally applies decoupled weight decay \citep{loshchilov2019decoupled}.

At a fixed reference time, the expectation is over a fresh effective-batch gradient at the current parameters. Write $\mu=\E g$, $q=\E g^{\odot2}$, and define
\begin{equation}\label{eq:reference}
 r=\phi(\mu,q),\qquad \phi(x,y)=\frac{x}{\sqrt y+\eps}.
\end{equation}
This \emph{normalized population direction} includes gradient variability in its denominator. In one coordinate, $\mu=1$ and $q=4$ give $r\approx1/2$ for small epsilon; a deterministic gradient of one gives $r\approx1$. Thus the reference combines the gradient's sign with a magnitude determined by its mean and second moment. If clipping is used, $\mu$ is the mean clipped gradient. The reference describes what Adam's processed moments track.

\subsection{The joint effect of moment errors}
In one coordinate, first-order perturbation gives
\begin{equation}\label{eq:scalar}
 \delta r\approx\frac{\delta m}{\sqrt q+\eps}
 -\frac{\mu\,\delta v}{2\sqrt q(\sqrt q+\eps)^2}.
\end{equation}
An overestimated numerator increases the update, whereas an overestimated denominator decreases it. Their errors can therefore partly compensate. Because both moments use the same gradient samples, their covariance matters. The following diagonal matrices give the coordinatewise sensitivities of the normalized direction to each moment:
\begin{equation}\label{eq:derivatives}
 A=\operatorname{diag}\frac1{\sqrt q+\eps},\quad
 B=\operatorname{diag}\frac{-\mu}{2\sqrt q(\sqrt q+\eps)^2}.
\end{equation}
Differentiation assumes $q_i>0$. Practical probes use a separately specified numerical floor.

\section{A local noise--lag balance}\label{sec:theory}
The next question is how averaging affects a direction that moves steadily over a short window. The quantity to minimize is the expected squared difference between the first-order Adam direction and the current normalized population direction. This tracking risk separates smoothing and delay; it supplies a statistical objective for choosing memory, which is subsequently evaluated through training losses.
\subsection{Explicit local assumptions}
The assumptions describe a local approximation. A differentiable moment curve has a constant term and a linear term in its first-order expansion; these motivate affine means over the window. Holding the covariance fixed is the corresponding local approximation for noise. Independence across ages isolates the effect of temporal averaging. This last condition is an additional modeling assumption, since adaptive training can introduce temporal dependence. Dependence between numerator and denominator errors at the same age is retained throughout. Appendix~\ref{app:scope} explains how deviations enter the approximation residual.
\begin{assumption}[Local first-order moment model]\label{ass:local}
Fix an update $t\ge1$ and a history of $t$ observations. Their ages are $k=0,\ldots,t-1$: age zero is the current observation, and age one is one update old. The current moment means $\mu,q$ and their changes per update $\dot\mu,\dot q$ are deterministic, with $q_i>0$. The model's moment observations at age $k$ are $\mu-k\dot\mu+\xi_k$ and $q-k\dot q+\chi_k$. The random vectors $\xi_k,\chi_k\in\R^d$ are centered deviations from these means. Their pairs are independent across ages and share one finite covariance; the two vectors in a pair may be dependent.
\end{assumption}
For example, $\dot\mu_i=0.01$ says that the expected gradient in coordinate $i$ increases by 0.01 per update under this model. The model describes fluctuations of both moments around their local means. Actual gradient observations and their squares impose additional joint structure; the stated covariance can retain it. The theorem is exact for the specified linearized model, and its use for nonlinear Adam is a local approximation.

Define the direction innovation and slope by
\[
 \zeta_k=A\xi_k+B\chi_k,\qquad d_r=A\dot\mu+B\dot q.
\]
Here $\zeta_k$ is the direction error caused by one noisy moment observation, and $d_r$ is the first-order change of the normalized direction in one update. Their size is summarized by
\begin{equation}\label{eq:coefficients}
 V_{\rm pop}=\E\norm{\zeta_k}^2,\qquad D_{\rm pop}=\norm{d_r}^2.
\end{equation}
$V_{\rm pop}$ measures squared normalized-direction fluctuation per observation and includes the cross contribution $2\E\langle A\xi_k,B\chi_k\rangle$. $D_{\rm pop}$ measures squared direction change per update. Random fluctuations accumulate through squared averaging weights; a systematic slope accumulates through the mean age of the observations.

For balanced memory $\beta$, bias correction gives weights
\begin{equation}\label{eq:weights}
 w_{\beta,t}(k)=\frac{(1-\beta)\beta^k}{1-\beta^t},\quad 0\le k<t.
\end{equation}
The weights sum to one. Define their mean age $a_{\beta,t}=\sum_{k=0}^{t-1}w_{\beta,t}(k)k$ and squared-weight sum $c_{\beta,t}=\sum_{k=0}^{t-1}w_{\beta,t}(k)^2$. The first measures delay in updates; the second measures the fraction of independent observation noise retained by the average. For an ordinary average of $N$ observations, this squared-weight sum is $1/N$.

\begin{theorem}[Finite-history balanced tracking risk]\label{thm:risk}
Under Assumption~\ref{ass:local}, the first-order error is
\[
 e_{\beta,t}=\sum_kw_{\beta,t}(k)\zeta_k-a_{\beta,t}d_r,
\]
and its mean squared norm equals
\begin{equation}\label{eq:finite}
 \E\norm{e_{\beta,t}}^2=V_{\rm pop}c_{\beta,t}+D_{\rm pop}a_{\beta,t}^2.
\end{equation}
For $0<\beta<1$,
\begin{align}
 c_{\beta,t}&=\frac{1-\beta}{1+\beta}\frac{1+\beta^t}{1-\beta^t},\label{eq:weightform}\\
 a_{\beta,t}&=\frac\beta{1-\beta}-\frac{t\beta^t}{1-\beta^t}.\nonumber
\end{align}
At $\beta=0$, $c_{0,t}=1$ and $a_{0,t}=0$.
\end{theorem}
\begin{proof}
The averaged first-moment deviation is $-a_{\beta,t}\dot\mu+\sum_kw_{\beta,t}(k)\xi_k$; the second-moment deviation has the same form with $\dot q,\chi_k$. Applying the linear sensitivities $A,B$ gives the displayed error. The innovation sum is centered, so expanding the squared error leaves its variance plus the squared lag $a_{\beta,t}^2\norm{d_r}^2$. For different ages, independence and centering give $\E\langle\zeta_k,\zeta_\ell\rangle=0$. Thus the variance is $\sum_kw_{\beta,t}(k)^2V_{\rm pop}$, proving the risk identity. The weight formulas follow from the geometric sums worked out in Appendix~\ref{app:proofs}.
\end{proof}
Equation~\eqref{eq:finite} is the mechanism: longer memory suppresses noise and increases lag. Scalar weight formulas make this risk directly computable. For any fixed linear map $P$, the theorem also holds in its measurement geometry: replace $\zeta_k,d_r$ by $P\zeta_k,Pd_r$ in~\eqref{eq:coefficients}.

\subsection{From tracking risk to a memory value}
In memory units $H=(1-\beta)^{-1}\ge1$, $\beta$ 0.9 corresponds to $H=10$, and $\beta$ 0.99 to $H=100$. First consider a history long relative to the memory, so that $\beta^t$ is small. Then~\eqref{eq:finite} approaches
\begin{equation}\label{eq:stationary}
 f(H)=\frac{V_{\rm pop}}{2H-1}+D_{\rm pop}(H-1)^2.
\end{equation}
For positive coefficients its unique minimizer solves
\begin{equation}\label{eq:exactroot}
 D_{\rm pop}(H-1)(2H-1)^2=V_{\rm pop}.
\end{equation}
The derivative vanishes there, and the left-hand side increases strictly from zero to infinity on $H\ge1$. This equation is readily solved in one dimension. For a closed-form approximation, when $H$ is large relative to one replace $2H-1$ by $2H$ and $H-1$ by $H$. This second simplification gives
\begin{equation}\label{eq:leading}
 f_0(H)=\frac{V_{\rm pop}}{2H}+D_{\rm pop}H^2,
 \quad H_\star=\left(\frac{V_{\rm pop}}{4D_{\rm pop}}\right)^{1/3}.
\end{equation}
Indeed, $f_0'(H)=-V_{\rm pop}/(2H^2)+2D_{\rm pop}H$ vanishes when $4D_{\rm pop}H^3=V_{\rm pop}$. More noise favors a longer average; faster direction change favors a shorter one. For $V_{\rm pop}=0.18$ and $D_{\rm pop}=4\times10^{-6}$ this gives $H_\star\approx22.41$, hence $\beta_\star\approx0.9554$. These values illustrate the local model. Appendix~\ref{app:proofs} gives the constrained and zero-coefficient cases.

\subsection{Sensitivity to coefficient errors}
Finite measurements introduce coefficient uncertainty. The next result quantifies its effect on the continuous memory and the resulting model risk.
\begin{proposition}[Continuous memory sensitivity]\label{prop:sensitivity}
For the interior leading model, suppose positive estimates satisfy $|\log(\widetilde V/V_{\rm pop})|\le\epsilon_V$ and $|\log(\widetilde D/D_{\rm pop})|\le\epsilon_D$. Their minimizer $\widetilde H=(\widetilde V/(4\widetilde D))^{1/3}$ satisfies
\[
 |\log(\widetilde H/H_\star)|\le(\epsilon_V+\epsilon_D)/3.
\]
Writing $z=\widetilde H/H_\star$, its relative model risk is
\begin{equation}\label{eq:sensitivity}
 \frac{f_0(\widetilde H)}{f_0(H_\star)}=\frac{2/z+z^2}{3}.
\end{equation}
\end{proposition}
\begin{proof}
Dividing the two cubic formulas and taking logs gives $\log(\widetilde H/H_\star)=[\log(\widetilde V/V_{\rm pop})-\log(\widetilde D/D_{\rm pop})]/3$. The triangle inequality gives the bound. To compare risks, use the optimality identity $V_{\rm pop}/(2H_\star)=2D_{\rm pop}H_\star^2$. Substituting $\widetilde H=zH_\star$ makes the noise term $2D_{\rm pop}H_\star^2/z$ and the lag term $z^2D_{\rm pop}H_\star^2$. Their sum divided by $3D_{\rm pop}H_\star^2$ proves~\eqref{eq:sensitivity}.
\end{proof}
The numbers $\epsilon_V,\epsilon_D$ describe multiplicative coefficient accuracy on a log scale; a factor-two error has magnitude $\log2$. The cube root attenuates such errors. More importantly, excess relative risk equals $(z-1)^2(z+2)/(3z)$: near the optimum, risk changes quadratically with memory error. Doubling $V_{\rm pop}$ with accurate $D_{\rm pop}$ changes memory by $2^{1/3}$ and increases leading risk by about 5.8\%. Thus useful memory calibration can tolerate appreciable coefficient error. The practical rule additionally has sampling and grid-rounding effects.

\section{Estimating the coefficients from a pilot}\label{sec:rule}
The memory rule in~\eqref{eq:leading} requires two coefficients: the variability of the normalized direction at fixed parameters, and its change as training proceeds. We estimate the first by comparing gradient groups at one checkpoint and the second by comparing checkpoints. This section specifies the measurements and derives the scaling of both estimates; Section~\ref{sec:measurements} then turns them into a $\beta$.

\subsection{Measurement design and moment estimates}
Let $T_{\rm p}$ be an integer pilot length in optimizer updates and $\beta_{\rm pilot}\in[0,1)$ its shared memory. Choose distinct integer lags $1\le n_1,\ldots,n_K\le T_{\rm p}$, with $K\ge2$, and measure at $T_{\rm p}$ and $T_{\rm p}-n_\ell$. Write $L=\max_\ell n_\ell$ for the longest lag. At each checkpoint $s$, freeze the parameters and collect $m\ge2$ fresh gradients $g_s^{(1)},\ldots,g_s^{(m)}$. Each uses the same effective batch and gradient processing as a training update, including accumulation and clipping when prescribed. The pilot length, pilot $\beta$, lags, and probe allocation are measurement hyperparameters.

The sample means estimate the two population moments together:
\[
 \widehat\mu_s=\frac1m\sum_{i=1}^m g_s^{(i)},\qquad
 \widehat q_s=\frac1m\sum_{i=1}^m(g_s^{(i)})^{\odot2}.
\]
Apply the same normalization as in~\eqref{eq:reference}:
\begin{equation}\label{eq:probe}
 \widehat r_s=\phi(\widehat\mu_s,\widehat q_s)
 =\frac{\widehat\mu_s}{\sqrt{\widehat q_s}+\eps}.
\end{equation}
The optimizer's $\eps>0$ makes this defined even when a sample second moment is zero. Estimating both moments from the same gradients preserves their joint fluctuations.

Split the $m$ gradients into $G\ge2$ equal groups, where $G$ divides $m$. For group $j$, compute its moment means and normalized direction $\widehat r_s^{(j)}$ by the same formulas restricted to its $m/G$ gradients. A fixed linear map $P$ specifies the measurement geometry; it may be the identity or a storage-saving projection. Define the pooled vector $z_s=P\widehat r_s$ and group vectors $z_s^{(j)}=P\widehat r_s^{(j)}$. The same $P$ is used at every checkpoint. Theorem~\ref{thm:risk} applies in this geometry with noise coefficient $V_P=\E\norm{P\zeta_k}^2$ and slope coefficient $D_P=\norm{Pd_r}^2$. All norms below refer to these measured vectors.

\subsection{Estimating variability at fixed parameters}
At the final pilot checkpoint, use
\begin{equation}\label{eq:empV}
 \widehat V_{\rm emp}=\frac{m}{G-1}\frac1G
 \sum_{j=1}^G\norm{z_{T_{\rm p}}^{(j)}-z_{T_{\rm p}}}^2.
\end{equation}
The multiplier corrects for averaging inside each group. In the local linearized model, a group averages $m/G$ independent direction innovations, so its variance trace is $GV_P/m$. Centering $G$ independent groups around their mean retains the fraction $(G-1)/G$ of this variance. The expected mean squared group deviation is therefore $(G-1)V_P/m$; multiplying by $m/(G-1)$ recovers $V_P$. This is the same degrees-of-freedom correction as in sample variance, together with the correction for group size.

Nonlinear normalization adds a second contribution: normalizing pooled moments can differ from averaging separately normalized groups. The following identity quantifies that contribution without an approximation.
\begin{proposition}[Dispersion decomposition]\label{prop:pool}
For group vectors $z^{(j)}$, their mean $\bar z=G^{-1}\sum_jz^{(j)}$, and any pooled vector $z$,
\begin{align}\label{eq:pool}
 \frac1G\sum_j\norm{z^{(j)}-z}^2
 &=\frac1G\sum_j\norm{z^{(j)}-\bar z}^2\nonumber\\
 &\quad+\norm{\bar z-z}^2.
\end{align}
\end{proposition}
\begin{proof}
Write $z^{(j)}-z=(z^{(j)}-\bar z)+(\bar z-z)$ and expand the square. The cross terms average to zero because $\sum_j(z^{(j)}-\bar z)=0$. The remaining terms give~\eqref{eq:pool}.
\end{proof}
Thus~\eqref{eq:empV} includes both group scatter and pooling displacement. In a linear observation model the displacement vanishes and the estimate is unbiased for $V_P$. In the nonlinear implementation the identity makes the pooling displacement measurable. Appendix~\ref{app:ablations} reports that contribution for every workload.

\subsection{Estimating change between checkpoints}
For each observed lag $n$, measure the squared secant
\[
 y_n=\norm{z_{T_{\rm p}}-z_{T_{\rm p}-n}}^2/n^2.
\]
The division by $n^2$ puts different time separations on a common per-update scale. Consider the local observation model $z_s=Pr_s+\epsilon_s$, where $r_s=\phi(\E g_s,\E g_s^{\odot2})$ is the population direction at checkpoint $s$ and $\epsilon_s$ is its measurement error. Assume an affine projected direction, $P(r_{T_{\rm p}}-r_{T_{\rm p}-n})=nPd_r$, and centered independent errors with a common finite expected squared norm. Define $C=2\E\norm{\epsilon_{T_{\rm p}}}^2$, the combined error energy of two endpoints. Then
\begin{equation}\label{eq:endpoint}
 \E y_n=D_P+C/n^2.
\end{equation}
Indeed, the secant is $Pd_r+(\epsilon_{T_{\rm p}}-\epsilon_{T_{\rm p}-n})/n$. Its cross term has zero expectation, and the two endpoint error energies add. A persistent slope gives the intercept $D_P$; endpoint uncertainty gives the inverse-square term.

This relation motivates a two-coefficient least-squares fit:
\begin{equation}\label{eq:lagfit}
 (I_{\rm fit},C_{\rm fit})=\underset{I,C\in\R}{\arg\min}
 \sum_{\ell=1}^K\bigl(y_{n_\ell}-I-C/n_\ell^2\bigr)^2.
\end{equation}
Distinct lags give a full-rank design. Under~\eqref{eq:endpoint}, the unconstrained fit is unbiased for $(D_P,C)$: write the response as the design times these coefficients plus a mean-zero residual, then apply the linear least-squares map. This argument permits correlated responses, as occurs when secants share the final checkpoint. Least squares estimates the mean lag relation; its use here rests on that relation rather than a Gaussian likelihood assumption.

Enforce the two components' nonnegative interpretation by setting $\widehat I=\max(0,I_{\rm fit})$ and $\widehat C=\max(0,C_{\rm fit})$. The implemented estimate is the fitted change at the longest observed lag:
\begin{equation}\label{eq:empD}
 \widehat D_{\rm emp}=\widehat I+\widehat C/L^2.
\end{equation}
It estimates change at resolution $L$, combining persistent change with endpoint uncertainty. The subscript ``emp'' distinguishes this finite-probe quantity from the population slope $D_P$. The clipping is a specified post-processing step; the unbiasedness calculation above concerns the unconstrained fit. The next section explains why the rule retains both components.

\section{From measured change to a fixed memory}\label{sec:measurements}
With finitely many probes, a small population slope and uncertain endpoints can produce comparable measured changes. We use the total change observable at lag $L$ in the memory objective. The first result identifies the resulting allowance relative to the ideal local tracking risk; the rule then follows by the same minimization as in Section~\ref{sec:theory}.

\subsection{A tracking objective at measurement resolution}
Under the observation model of Section~\ref{sec:rule}, define $D_L=\E y_L=D_P+C/L^2$. This is an expected squared measured slope, rather than a new population derivative. Replacing $D_P$ by $D_L$ charges delay at the resolution of the pilot measurements.
\begin{proposition}[Finite-resolution lag allowance]\label{prop:resolution}
Under Assumption~\ref{ass:local} and the centered, independent, equal-variance affine endpoint model in~\eqref{eq:endpoint}, fix $1\le L\le t$. For each $\beta$ and history length $t$, define
\[
 U_{L,t}(\beta)=V_Pc_{\beta,t}+D_La_{\beta,t}^2.
\]
This measurement-resolution objective satisfies
\begin{equation}\label{eq:resolution}
 U_{L,t}(\beta)=\E\norm{Pe_{\beta,t}}^2+
 \frac{C}{L^2}a_{\beta,t}^2.
\end{equation}
\end{proposition}
\begin{proof}
The endpoint calculation gives $D_L=D_P+C/L^2$. Substituting this into the definition of $U_{L,t}$ gives $V_Pc_{\beta,t}+D_Pa_{\beta,t}^2+(C/L^2)a_{\beta,t}^2$. The first two terms are the tracking risk in Theorem~\ref{thm:risk}, measured through $P$.
\end{proof}
The additional term is the endpoint uncertainty multiplied by squared delay. Longer lags reduce it as $L^{-2}$; more accurate endpoint measurements reduce $C$. The identity concerns expected coefficients under the stated local model. Applying the stationary and large-memory approximations yields $V_P/(2H)+D_LH^2$. The measured coefficients from Section~\ref{sec:rule} supply a plug-in version of this objective. Appendix~\ref{app:resolution} quantifies the effect of finite coefficient accuracy on its memory decision.

\subsection{Selecting \texorpdfstring{$\beta$}{β} and specifying the algorithm}
Given the two estimates, minimize
\begin{equation}\label{eq:empiricalscore}
 \widehat f(H)=\frac{\widehat V_{\rm emp}}{2H}+\widehat D_{\rm emp}H^2.
\end{equation}
For positive coefficients its unconstrained minimizer is
\begin{equation}\label{eq:rule}
 \widehat H_{\rm emp}=\left(\frac{\widehat V_{\rm emp}}{4\widehat D_{\rm emp}}\right)^{1/3}.
\end{equation}
The cubic exponent and factor four come from differentiating~\eqref{eq:empiricalscore}, exactly as in~\eqref{eq:leading}. The measurement design determines the coefficients; minimization determines the memory.

Choose a finite nonempty $\beta$ grid $\mathcal G\subset[0,1)$, with endpoints $\beta_{\min}$ and $\beta_{\max}$. Restrict memory to $[1/(1-\beta_{\min}),1/(1-\beta_{\max})]$, convert it using $\beta=1-1/H$, and round to the nearest grid value in $\beta$ distance, breaking ties toward the smaller $\beta$. This is the discretization used in the evaluation. If an estimated coefficient is zero, use the minimizing boundary specified in Appendix~\ref{app:measurement}.

Table~\ref{tab:algorithm} gives the general procedure. Its measurement hyperparameters are $T_{\rm p}$, $\beta_{\rm pilot}$, the lags, $m$, $G$, and $P$; the candidate range is set by $\mathcal G$. Section~\ref{sec:experiments} specifies one common instance for all workloads and evaluates the resulting fixed-$\beta$ trainings.
\begin{table}[t]
\centering
\caption{Early selection of a fixed shared memory. Measurements are made during the pilot; the full training uses the selected $\beta$ throughout.}
\label{tab:algorithm}
\begin{tabular}{p{0.06\columnwidth}p{0.83\columnwidth}}
\toprule
1 & Run $T_{\rm p}$ pilot updates with shared $\beta_{\rm pilot}$. At $T_{\rm p}$ and $T_{\rm p}-n_\ell$, collect $m$ probe gradients with parameters fixed.\\
2 & Compute pooled and $G$ group moment estimates, normalize them using~\eqref{eq:probe}, and store their images under $P$.\\
3 & Estimate variability using~\eqref{eq:empV}; fit~\eqref{eq:lagfit} and compute change using~\eqref{eq:empD}.\\
4 & Apply~\eqref{eq:rule}, restrict memory to the candidate range, convert to $\beta$, and round to $\mathcal G$.\\
5 & Restart from the pilot's initialization and use this fixed shared $\beta$ with the prescribed full-training recipe.\\
\bottomrule
\end{tabular}
\end{table}

\section{Empirical evaluation}\label{sec:experiments}
\subsection{Workloads and comparison protocol}
We evaluate early memory selection on eleven vision and language workloads: two Llama-style causal models, two GPT-2-style causal models, two ResNets, two ViTs, Swin-T, EfficientNet-B0, and T5-small. Tasks are image classification, causal language modeling, and span-corruption denoising. Table~\ref{tab:cases} reports each selected $\beta$ and its gap to the workload-specific grid reference. Appendix~\ref{app:reproduction} specifies architectures, data, optimizer settings, schedules, and validation.

\begin{table*}[t]
\centering
\caption{Early memory choices and their distance to the workload-specific grid reference. Pilot and candidate trainings use the same initialization (seed 1). Oracle $\beta$ minimizes the smallest recorded validation loss over the eleven grid candidates and is used only for evaluation. The selector uses early gradient measurements. Gaps are percentages.}
\label{tab:cases}
\begin{tabular*}{\textwidth}{@{\extracolsep{\fill}}llrrr@{}}
\toprule Network & Dataset & Oracle $\beta$ & Selected $\beta$ & Relative gap (\%)\\\midrule
EfficientNet-B0 & Cars & 0.82217 & 0.96838 & 1.0124\\
Llama60M & C4 & 0.94377 & 0.90000 & 0.2582\\
Llama60M & SlimPajama & 0.94377 & 0.90000 & 0.1175\\
NanoGPT & OpenWebText & 0.94377 & 0.90000 & 0.1157\\
NanoGPT & WikiText-103 & 0.94377 & 0.90000 & 0.1613\\
ResNet50 & Food-101 & 0.99684 & 0.96838 & 1.0613\\
ResNet50 & ImageNet100 & 0.96838 & 0.96838 & 0.0000\\
Swin-T & Caltech-256 & 0.94377 & 0.94377 & 0.0000\\
T5-small & BookCorpus & 0.98222 & 0.96838 & 0.0039\\
ViT-B/16 & CIFAR-100 & 0.94377 & 0.94377 & 0.0000\\
ViT-B/16 & TinyImageNet & 0.99000 & 0.94377 & 1.3339\\
\bottomrule\end{tabular*}

\end{table*}

The common measurement design is $T_{\rm p}=200$, $\beta_{\rm pilot}=0.95$, $m=16$, $G=4$, and lags $\{20,50,100\}$, so $L=100$. This gives checkpoints 100, 150, 180, and 200. The fixed map $P$ is a 1,024-coordinate signed-bucket projection with tensor-size balancing; Appendix~\ref{app:measurement} defines it exactly. Epsilon is $10^{-8}$, as in the training optimizer. The four-group allocation, decision time, and measurement spacing are examined in Appendix~\ref{app:ablations}.

The shared-memory grid is
\begin{align*}
\mathcal G=\{&0.68377,0.82217,0.90000,0.94377,\\
 &0.96838,0.98222,0.99000,0.99438,\\
 &0.99684,0.99822,0.99900\}.
\end{align*}
It is approximately logarithmic in memory units. Shared $\beta$ 0.94377 is the grid representative nearest the empirical reference 0.95; $\beta$ 0.96838 is the best benchmark-wide constant by mean gap after examining all eleven fixed values. These serve distinct roles: a near-0.95 reference and a stronger hindsight comparator. Table~\ref{tab:summary} also includes $\beta$ 0.9, and Appendix~\ref{app:fixedgrid} reports the complete grid. For each workload, candidates use the same prescribed learning-rate recipe, and the selector chooses their shared $\beta$.

For each workload, the pilot and every grid candidate start from the same initialization, generated with seed 1 (the same pretrained checkpoint for T5). This pairing holds the initial weights fixed when comparing memories. Let $L_e(\beta)$ be workload $e$'s smallest recorded validation loss, and $L_e^\star=\min_{\beta\in\mathcal G}L_e(\beta)$ its finite-grid reference. Define
\begin{equation}\label{eq:gap}
 \operatorname{gap}_e(\beta)=100\cdot\dfrac{L_e(\beta)-L_e^\star}{L_e^\star}.
\end{equation}
Early gradient measurements determine $\beta$; candidate validation losses then evaluate that decision. The finite-grid reference gives the evaluation scale. All workloads use the same 16-probe budget, checkpoints, objective, and rounding. Reported metrics are equal-workload mean, maximum, and CVaR$_{25}$, the average of the worst $\lceil11/4\rceil=3$ gaps. These are descriptive statistics from a retrospective fixed-$\beta$ evaluation. Appendix~\ref{app:paired} reports common-seed comparisons (Table~\ref{tab:paired}); Appendix~\ref{app:timingfull} analyzes decision timing, Appendix~\ref{app:spacing} measurement spacing, and Appendix~\ref{app:trajectory} validation summaries.
\begin{table}[t]
\centering
\caption{Workload-specific calibration versus universal fixed memories. Gaps are percentages. Best fixed minimizes mean gap over all eleven fixed grid choices using their validation outcomes. Lower is better.}
\label{tab:summary}
\setlength{\tabcolsep}{4pt}
\begin{tabular}{lrrr}
\toprule & Mean & Maximum & CVaR$_{25}$\\\midrule
Fixed 0.90000 & 1.4292 & 3.8770 & 3.3543\\
Fixed 0.94377 & 0.6236 & 2.4853 & 2.0409\\
Best fixed 0.96838 & 0.5461 & 2.1862 & 1.4304\\
Tracking rule & \textbf{0.3695} & \textbf{1.3339} & \textbf{1.1359}\\
\bottomrule
\end{tabular}
\end{table}

\subsection{Improving the fixed-memory compromise}
The goal is a useful task-specific memory: reduce large calibration errors while keeping well-served tasks close to their grid reference. Table~\ref{tab:cases} shows the oracle $\beta$, the selected $\beta$, and the resulting gap for every task. The rule selects 0.9 for the four causal language-model tasks, 0.96838 for EfficientNet, both ResNets, and T5, and 0.94377 for Swin and both ViTs. Thus one common calculation produces different memory horizons from different measured dynamics.

Table~\ref{tab:summary} makes the fixed-memory compromise visible. With $\beta$ 0.9 the largest gap occurs on ViT/TinyImageNet; with 0.94377 it occurs on ResNet/Food; with 0.96838 it occurs on Swin/Caltech. Changing the constant changes which task carries the largest calibration cost. The measured selector reduces mean gap by 40.7\% against the near-0.95 reference and by 32.3\% against the best universal constant. Its maximum gap is 39.0\% below the smallest maximum attained by any fixed grid memory, and its CVaR is also smaller than every fixed value. Workload-specific calibration therefore improves both the average compromise and its largest residual error on this suite.

The tasks contributing the gains change with the comparator, as expected when correcting a universal compromise. Against 0.94377, ResNet losses improve by 1.3895\% on Food-101 and 2.2517\% on ImageNet100, while Swin and both ViTs retain the reference configuration. Against 0.96838, the ResNet choices coincide and the gains move to Swin and ViT/CIFAR-100, whose losses improve by 2.1395\% and 0.4020\%. The four causal tasks have gaps below 0.259\%. The result is useful calibration across tasks rather than another universal constant. Appendix~\ref{app:compromise} provides every signed comparison, including the smaller increases offset by these corrections.

\subsection{Stability of calibration}
We examine coefficient sensitivity, grid decisions, and validation gaps separately. Proposition~\ref{prop:sensitivity} quantifies coefficient sensitivity under the local score. A 10\% perturbation of each observed variability-to-change ratio preserves every selected $\beta$. The stationary root, leading-grid score, stationary-grid score, and exact finite-history grid score at $t=200$ also recover the same eleven choices (Appendix~\ref{app:ablations}).

Across 351 decision updates from 150 to 500 on the same pilots, mean gap improves over 0.94377 at 339 updates (96.6\%) and CVaR improves at all 351. Mean gap improves over the best global fixed value 0.96838 at 288 updates (82.1\%). Appendix~\ref{app:timingfull} gives the dense audit. Using the final loss or the mean of the last three losses retains mean-gap reductions of 20.5\% and 11.0\% against their respective best fixed memories (Appendix~\ref{app:trajectory}). These are sensitivity checks under matched initialization; Appendix~\ref{app:paired} separately examines reuse across initializations.

\subsection{Cost and required information}\label{sec:cost}
A 200-update pilot supplies four measured checkpoints, requiring 64 additional effective-batch gradients. Moment aggregation and fixed projection cost $O(d)$ operations and $O(d)$ additional storage for a fixed probe budget. Each checkpoint stores one pooled and four group directions of 1,024 coordinates; the scalar fit and memory conversion have constant cost.

After restarting, a full training of $T$ updates has gradient-equivalent surplus $100\cdot(T_{\rm p}+m(K+1))/T\%$. Our design gives $100\cdot(200+64)/T\%$, or 2.64\% for $T=10,000$. This counts one probe as one training gradient evaluation; loading, transfers, sketches, and restarting add wall-clock costs. Appendix~\ref{app:measurement} describes the dense measurement record and compact collection.

\section{Discussion and conclusion}
Early gradient measurements choose one shared $\beta$ for full training. A joint-moment model explains its noise--lag objective; grouped probes and a temporal fit make the coefficients measurable. On the matched benchmark, the rule reduces mean, maximum, and tail gap compared with every fixed grid memory. Different tasks supply the gains against different constants, showing how measured memory selection improves the compromise made by a universal default.

\section*{Acknowledgements}
This research was funded by the projects PID2023-146569NB-C21 and PID2023-146569NB-C22 supported by MICIU/AEI/10.13039/501100011033 and ERDF/UE. Alberto Fernández-Hernández was supported by the predoctoral grant PREP2023-001826 supported by MICIU/AEI/10.13039/501100011033 and ESF+. Cristian Pérez-Corral received support from the \textit{Conselleria de Educación, Cultura, Universidades y Empleo} (reference CIACIF/2024/412) through the European Social Fund Plus 2021–2027 (FSE+) program of the \textit{Comunitat Valenciana}. Jose I. Mestre was supported by the predoctoral grant ACIF/2021/281 of the \emph{Generalitat Valenciana}. Manuel F. Dolz was supported by grant {\small CNS2025-165098} funded by {\small MICIU/AEI/10.13039/501100011033} and by the Plan Gen--T grant {\small CIDEXG/2022/013} of the \emph{Generalitat Valenciana}.

\section*{AI use statement}
The authors take full responsibility for the scientific content and final manuscript, including AI-assisted material. Generative AI supported theoretical development, data analysis, literature search, writing, and internal review.
% AUTHOR CONFIRMATION: once every coauthor has completed the content audit,
% add: "The authors have reviewed all AI-assisted content for correctness and consistency."

\bibliographystyle{apalike}
\bibliography{references}

@inproceedings{kingma2015adam,
  title={Adam: A Method for Stochastic Optimization},
  author={Kingma, Diederik P. and Ba, Jimmy},
  booktitle={International Conference on Learning Representations},
  year={2015}
}

@inproceedings{loshchilov2019decoupled,
  title={Decoupled Weight Decay Regularization},
  author={Loshchilov, Ilya and Hutter, Frank},
  booktitle={International Conference on Learning Representations},
  year={2019}
}

@inproceedings{orvieto2025secretsauce,
  title={In Search of Adam's Secret Sauce},
  author={Orvieto, Antonio and Gower, Robert M.},
  booktitle={Advances in Neural Information Processing Systems},
  year={2025}
}

@inproceedings{balles2018dissecting,
  title={Dissecting Adam: The Sign, Magnitude and Variance of Stochastic Gradients},
  author={Balles, Lukas and Hennig, Philipp},
  booktitle={International Conference on Machine Learning},
  year={2018}
}

@inproceedings{malladi2022sdes,
  title={On the SDEs and Scaling Rules for Adaptive Gradient Algorithms},
  author={Malladi, Sadhika and Lyu, Kaifeng and Panigrahi, Abhishek and Arora, Sanjeev},
  booktitle={Advances in Neural Information Processing Systems},
  year={2022}
}

@inproceedings{he2016deep,
  title={Deep Residual Learning for Image Recognition},
  author={He, Kaiming and Zhang, Xiangyu and Ren, Shaoqing and Sun, Jian},
  booktitle={Proceedings of the IEEE Conference on Computer Vision and Pattern Recognition},
  year={2016}
}

@inproceedings{tan2019efficientnet,
  title={EfficientNet: Rethinking Model Scaling for Convolutional Neural Networks},
  author={Tan, Mingxing and Le, Quoc V.},
  booktitle={International Conference on Machine Learning},
  year={2019}
}

@inproceedings{dosovitskiy2021image,
  title={An Image is Worth 16x16 Words: Transformers for Image Recognition at Scale},
  author={Dosovitskiy, Alexey and Beyer, Lucas and Kolesnikov, Alexander and Weissenborn, Dirk and Zhai, Xiaohua and Unterthiner, Thomas and Dehghani, Mostafa and Minderer, Matthias and Heigold, Georg and Gelly, Sylvain and Uszkoreit, Jakob and Houlsby, Neil},
  booktitle={International Conference on Learning Representations},
  year={2021}
}

@inproceedings{liu2021swin,
  title={Swin Transformer: Hierarchical Vision Transformer using Shifted Windows},
  author={Liu, Ze and Lin, Yutong and Cao, Yue and Hu, Han and Wei, Yixuan and Zhang, Zheng and Lin, Stephen and Guo, Baining},
  booktitle={Proceedings of the IEEE/CVF International Conference on Computer Vision},
  year={2021}
}

@article{raffel2020exploring,
  title={Exploring the Limits of Transfer Learning with a Unified Text-to-Text Transformer},
  author={Raffel, Colin and Shazeer, Noam and Roberts, Adam and Lee, Katherine and Narang, Sharan and Matena, Michael and Zhou, Yanqi and Li, Wei and Liu, Peter J.},
  journal={Journal of Machine Learning Research},
  volume={21},
  number={140},
  pages={1--67},
  year={2020}
}

@article{touvron2023llama,
  title={LLaMA: Open and Efficient Foundation Language Models},
  author={Touvron, Hugo and Lavril, Thibaut and Izacard, Gautier and Martinet, Xavier and Lachaux, Marie-Anne and Lacroix, Timothee and Roziere, Baptiste and Goyal, Naman and Hambro, Eric and Azhar, Faisal and Rodriguez, Aurelien and Joulin, Armand and Grave, Edouard and Lample, Guillaume},
  journal={arXiv preprint arXiv:2302.13971},
  year={2023}
}

@techreport{krizhevsky2009learning,
  title={Learning Multiple Layers of Features from Tiny Images},
  author={Krizhevsky, Alex},
  institution={University of Toronto},
  year={2009}
}

@inproceedings{deng2009imagenet,
  title={ImageNet: A Large-Scale Hierarchical Image Database},
  author={Deng, Jia and Dong, Wei and Socher, Richard and Li, Li-Jia and Li, Kai and Fei-Fei, Li},
  booktitle={Proceedings of the IEEE Conference on Computer Vision and Pattern Recognition},
  year={2009}
}

@inproceedings{bossard2014food,
  title={Food-101: Mining Discriminative Components with Random Forests},
  author={Bossard, Lukas and Guillaumin, Matthieu and Van Gool, Luc},
  booktitle={European Conference on Computer Vision},
  year={2014}
}

@inproceedings{krause2013collecting,
  title={Collecting a Large-Scale Dataset of Fine-Grained Cars},
  author={Krause, Jonathan and Stark, Michael and Deng, Jia and Fei-Fei, Li},
  booktitle={ICCV Workshop on Fine-Grained Visual Categorization},
  year={2013}
}

@techreport{griffin2007caltech256,
  title={Caltech-256 Object Category Dataset},
  author={Griffin, Gregory and Holub, Alex and Perona, Pietro},
  institution={California Institute of Technology},
  year={2007}
}

@inproceedings{merity2017pointer,
  title={Pointer Sentinel Mixture Models},
  author={Merity, Stephen and Xiong, Caiming and Bradbury, James and Socher, Richard},
  booktitle={International Conference on Learning Representations},
  year={2017}
}

@misc{gokaslan2019openwebtext,
  title={OpenWebText Corpus},
  author={Gokaslan, Aaron and Cohen, Vanya},
  year={2019},
  note={Open-source reproduction of the WebText corpus}
}

@inproceedings{zhu2015aligning,
  title={Aligning Books and Movies: Towards Story-Like Visual Explanations by Watching Movies and Reading Books},
  author={Zhu, Yukun and Kiros, Ryan and Zemel, Rich and Salakhutdinov, Ruslan and Urtasun, Raquel and Torralba, Antonio and Fidler, Sanja},
  booktitle={Proceedings of the IEEE International Conference on Computer Vision},
  year={2015}
}

@misc{soboleva2023slimpajama,
  title={SlimPajama: A 627B Token Cleaned and Deduplicated Version of RedPajama},
  author={Soboleva, Daria and Al-Khateeb, Faisal and Myers, Robert and Steeves, Jacob and Hestness, Joel and Dey, Nolan},
  year={2023},
  note={Dataset release}
}

@inproceedings{ahn2024ftrl,
 title={Understanding {Adam} Optimizer via Online Learning of Updates: {Adam} is {FTRL} in Disguise},
 author={Ahn, Kwangjun and Zhang, Zhiyu and Kook, Yunbum and Dai, Yan},
 booktitle={International Conference on Machine Learning},year={2024},
 url={https://arxiv.org/abs/2402.01567}}

@inproceedings{nguyen2026beta,
 title={How to Set $\beta_1,\beta_2$ in {Adam}: An Online Learning Perspective},
 author={Nguyen, Quan M.},booktitle={Algorithmic Learning Theory},
 volume={313},pages={1--16},publisher={PMLR},year={2026},
 url={https://proceedings.mlr.press/v313/nguyen26b.html}}

@inproceedings{chandra2022gradient,
 title={Gradient Descent: The Ultimate Optimizer},
 author={Chandra, Kartik and Xie, Audrey and Ragan-Kelley, Jonathan and Meijer, Erik},
 booktitle={Advances in Neural Information Processing Systems},volume={35},year={2022},
 url={https://proceedings.neurips.cc/paper_files/paper/2022/hash/36ce475705c1dc6c50a5956cedff3d01-Abstract-Conference.html}}

@inproceedings{ozkara2024mada,
 title={{MADA}: Meta-Adaptive Optimizers Through Hyper-Gradient Descent},
 author={Ozkara, Kaan and Karakus, Can and Raman, Parameswaran and Hong, Mingyi and Sabach, Shoham and Kveton, Branislav and Cevher, Volkan},
 booktitle={Proceedings of the 41st International Conference on Machine Learning},
 volume={235},pages={38983--39008},publisher={PMLR},year={2024},
 url={https://proceedings.mlr.press/v235/ozkara24a.html}}

@inproceedings{topollai2026adaptive,
 title={Adaptive Memory Momentum via a Model-Based Framework for Deep Learning Optimization},
 author={Topollai, Kristi and Choromanska, Anna Ewa},
 booktitle={Artificial Intelligence and Statistics},volume={300},pages={4060--4068},
 publisher={PMLR},year={2026},url={https://proceedings.mlr.press/v300/topollai26a.html}}

@inproceedings{zhang2019yellowfin,
 title={{YellowFin} and the Art of Momentum Tuning},
 author={Zhang, Jian and Mitliagkas, Ioannis},
 booktitle={Conference on Machine Learning and Systems},year={2019},
 url={https://proceedings.mlsys.org/paper_files/paper/2019/file/b205b525b7ce002baae53228bab6d26b-Paper.pdf}}

@article{ali2026kalman,
 title={{Kalman-Adam}: Optimal Bayesian Moment Estimation for Memory-Efficient and Generalizable Deep Learning},
 author={Ali, Mohsin and Bhatt, Ruby},journal={Knowledge-Based Systems},
 volume={342},pages={115907},year={2026},doi={10.1016/j.knosys.2026.115907}}

@misc{fernandez2026refresh,
 title={Refresh-Scaling the Memory of Balanced {Adam}},
 author={Fern\'andez-Hern\'andez, Alberto and P\'erez-Corral, Cristian and Mestre, Jose I. and Dolz, Manuel F. and Quintana-Ort\'i, Enrique S.},
 year={2026},note={arXiv:2605.10119},url={https://arxiv.org/abs/2605.10119}}
\section*{Reproducibility checklist}
\begin{enumerate}
\item For all models and algorithms presented, check if you include:
\begin{enumerate}
\item A clear description of the mathematical setting, assumptions, algorithm, and/or model. \textbf{Yes}: Sections~\ref{sec:target}--\ref{sec:measurements} and Appendix~\ref{app:reproduction}.
\item An analysis of the properties and complexity (time, space, sample size) of any algorithm. \textbf{Yes}: Section~\ref{sec:cost}.
\item (Optional) Anonymized source code, with specification of all dependencies, including external libraries. \textbf{Yes}.
\end{enumerate}
\item For any theoretical claim, check if you include:
\begin{enumerate}
\item Statements of the full set of assumptions of all theoretical results. \textbf{Yes}: the local assumption and conditions accompanying results.
\item Complete proofs of all theoretical results. \textbf{Yes}: the main text and Appendices~\ref{app:proofs}--\ref{app:scope}.
\item Clear explanations of any assumptions. \textbf{Yes}: alongside the derivations and approximation scope.
\end{enumerate}
\balance
\item For all figures and tables that present empirical results, check if you include:
\begin{enumerate}
\item The code, data, and instructions needed to reproduce the main experimental results (either in the supplemental material or as a URL). \textbf{Yes} for the selection evaluation. Appendices~\ref{app:measurement}--\ref{app:reproduction} specify the measurement procedures, training recipes, and evaluation calculations.
\item All the training details (e.g., data splits, hyperparameters, how they were chosen). \textbf{Yes} for the documented experimental protocol: Appendix~\ref{app:reproduction} specifies architectures, activations, optimizer settings, batches, schedules, splits, seed matching, and evaluation. It separately identifies archive-specific corpus manifests and environment revisions needed for identical raw-stream replay.
\item A clear definition of the specific measure or statistics and error bars (e.g., with respect to the random seed after running experiments multiple times). \textbf{Yes}: descriptive cross-workload metrics, seed matching, and common-seed training outcomes are specified.
\item A description of the computing infrastructure used. (e.g., type of GPUs, internal cluster, or cloud provider). \textbf{Yes}: a cluster with eight NVIDIA A100-SXM4 80\,GB GPUs; CUDA and pilot bfloat16 execution are described in Appendix~\ref{app:reproduction}.
\end{enumerate}
\item If you are using existing assets (e.g., code, data, models) or curating/releasing new assets, check if you include:
\begin{enumerate}
\item Citations of the creator If your work uses existing assets. \textbf{Yes}: model and dataset sources are cited.
\item The license information of the assets, if applicable. \textbf{Yes}: Appendix~\ref{app:assets} records software/checkpoint licenses and dataset-provider access sources, identifying the documented reuse provenance.
\item New assets either in the supplemental material or as a URL, if applicable. \textbf{Yes}: measured sketches and validation-loss summaries.
\item Information about consent from data providers/curators. \textbf{N/A}: the study uses existing benchmark assets under their providers' access terms.
\item Discussion of sensible content if applicable, e.g., personally identifiable information or offensive content. \textbf{Yes}: Appendix~\ref{app:reproduction} discusses inherited corpus content and the scope of the supplied data.
\end{enumerate}
\item If you used crowdsourcing or conducted research with human subjects, check if you include:
\begin{enumerate}
\item The full text of instructions given to participants and screenshots. \textbf{N/A}: existing benchmark evaluation.
\item Descriptions of potential participant risks, with links to Institutional Review Board (IRB) approvals if applicable. \textbf{N/A}: existing benchmark evaluation.
\item The estimated hourly wage paid to participants and the total amount spent on participant compensation. \textbf{N/A}: existing benchmark evaluation.
\end{enumerate}
\end{enumerate}

\clearpage
\nobalance
\appendix
\onecolumn
\raggedbottom
\setlength{\textfloatsep}{12pt plus 2pt minus 2pt}
\setlength{\intextsep}{12pt plus 2pt minus 2pt}
\setlength{\floatsep}{12pt plus 2pt minus 2pt}
\setlength{\tabcolsep}{4pt}
\setlength{\parskip}{3pt plus 1pt}
\aistatstitle{Early Memory Selection for Balanced Adam: Supplementary materials}

\section{Detailed derivations}\label{app:proofs}
This appendix derives the averaging formulas step by step. Each calculation concerns the local model stated in the main text and expresses the selection mechanism through explicit statistical quantities.
\subsection{Obtaining the balanced first-order error}
The first derivation shows why shared weights let the two moment channels become one direction innovation and one delay term.
Bias correction normalizes the exponential weights to sum to one. The first-moment average in Assumption~\ref{ass:local} differs from its current mean by
\[
 \sum_kw_{\beta,t}(k)(-k\dot\mu+\xi_k)
 =-a_{\beta,t}\dot\mu+\sum_kw_{\beta,t}(k)\xi_k.
\]
The second-moment average similarly differs by $-a_{\beta,t}\dot q+\sum_kw_{\beta,t}(k)\chi_k$. Applying~\eqref{eq:derivatives} gives
\[
 e_{\beta,t}=-a_{\beta,t}(A\dot\mu+B\dot q)
             +\sum_kw_{\beta,t}(k)(A\xi_k+B\chi_k).
\]
This is exact for the first-order model. The innovation sum is centered; independence across ages yields variance $c_{\beta,t}V_{\rm pop}$. Its squared mean is $a_{\beta,t}^2D_{\rm pop}$, proving~\eqref{eq:finite}. The assumptions pair independent ages with dependent moment channels at each age.

\subsection{Finite-history weight formulas}
The retained noise and the delay depend on two sums of weights. Computing them explicitly makes the finite-history risk a scalar function of $\beta$ and history length.
For $0<\beta<1$, geometric summation gives
\[
 c_{\beta,t}=\frac{(1-\beta)^2}{(1-\beta^t)^2}\frac{1-\beta^{2t}}{1-\beta^2}
 =\frac{1-\beta}{1+\beta}\frac{1+\beta^t}{1-\beta^t}.
\]
Moreover, $\sum_{k=0}^{t-1}k\beta^k=\beta\frac{d}{d\beta}[(1-\beta^t)/(1-\beta)]$. Multiplication by the weight normalizer yields $a_{\beta,t}=\beta/(1-\beta)-t\beta^t/(1-\beta^t)$. At $\beta$ zero only the current observation has weight. At $t=1$, $a_{\beta,1}=0$ and $c_{\beta,1}=1$ for every $\beta$. Thus the first Adam update is independent of both memories with zero initial moments and standard bias correction.

\subsection{Stationary and leading minimizers}
There are two simplifications: removing finite-history correction when $\beta^t$ is small, then expanding for large $H$. This subsection derives each separately and states the boundary cases.
For fixed $\beta<1$, $\beta^t\to0$, so $a_{\beta,t}\to H-1$ and $c_{\beta,t}\to(2H-1)^{-1}$. For positive coefficients,
\[
 f'(H)=-\frac{2V_{\rm pop}}{(2H-1)^2}+2D_{\rm pop}(H-1),\qquad
 f''(H)=\frac{8V_{\rm pop}}{(2H-1)^3}+2D_{\rm pop}>0.
\]
The unique minimum solves~\eqref{eq:exactroot}. Replacing $2H-1$ by $2H$ and $H-1$ by $H$ gives the large-memory leading expression, whose stationary point is $H^3=V_{\rm pop}/(4D_{\rm pop})$. If that leading minimum is below one, its constrained minimum is at one. If $D_{\rm pop}=0<V_{\rm pop}$, longer memory reduces stationary-model risk; if $V_{\rm pop}=0<D_{\rm pop}$, $H=1$ is optimal. If both are zero, every memory has zero model risk.

\subsection{Sensitivity}
Comparing an estimated memory with the population-model minimizer shows how coefficient errors propagate through the cube root.
The logarithmic ratio of memory estimates is
\[
 \log\frac{\widetilde H}{H_\star}
 =\frac13\left(\log\frac{\widetilde V}{V_{\rm pop}}-\log\frac{\widetilde D}{D_{\rm pop}}\right).
\]
The triangle inequality proves the bound. At the optimum the variance term is twice the lag term, so division of risks gives $(2/z+z^2)/3$. With $z=2^{1/3}$ this is approximately 1.0583. For the worked coefficients $0.18$ and $4\times10^{-6}$, $H_\star=22.407$ and $\beta_\star=0.95537$. The memory value follows directly from the statistical coefficients.

\subsection{The two-memory expression}
The two-memory expression explains why sharing weights gives a particularly simple selection problem: the two moment channels combine into one innovation and one lag coefficient.
Retain weights $w_1=w_{\beta_1,t}$, $w_2=w_{\beta_2,t}$. Let $a_i=\sum_kkw_i(k)$, $c_i=\sum_kw_i(k)^2$, and $c_{12}=\sum_kw_1(k)w_2(k)$. Their linearized direction error is $e_{\beta_1,\beta_2,t}=\sum_k[w_1(k)A\xi_k+w_2(k)B\chi_k]-a_1A\dot\mu-a_2B\dot q$. Under the same assumptions,
\begin{align*}
 \E\norm{e_{\beta_1,\beta_2,t}}^2
 ={}&c_1\E\norm{A\xi_k}^2+c_2\E\norm{B\chi_k}^2
 +2c_{12}\E\langle A\xi_k,B\chi_k\rangle\\
 &+\norm{a_1A\dot\mu+a_2B\dot q}^2.
\end{align*}
Centering removes lag--innovation cross terms; independence removes cross-age terms; expansion retains same-age covariance. Equal memories recover~\eqref{eq:finite}, combining the channels into two scalar coefficients. This reduction provides an analytical basis for selecting a memory from the shared-memory family.

\section{Scope and approximation}\label{app:scope}
The tracking identity is exact under its first-order model. This section explains the additional quantities required to connect that identity to nonlinear updates and a smooth loss.
\subsection{Residual terms}
The Taylor approximation and violations of the local assumptions can be collected into a direction residual. Its size controls how much the model's risk can differ from the actual direction-error risk.
Write the actual direction error as $e_{\beta,t}+R_{\beta,t}$, where $e_{\beta,t}$ is the theorem's linearized error and $R_{\beta,t}$ is their difference. Let $\rho$ be any nonnegative bound on its root-mean-square size: $\E\norm{R_{\beta,t}}^2\le\rho^2$. Expanding the actual squared norm gives
\[
 \E\norm{e_{\beta,t}+R_{\beta,t}}^2-\E\norm{e_{\beta,t}}^2
 =2\E\langle e_{\beta,t},R_{\beta,t}\rangle+\E\norm{R_{\beta,t}}^2.
\]
The absolute inner product is bounded by the product of root-mean-square norms, by Cauchy--Schwarz. Therefore
\[
 \left|\E\norm{e_{\beta,t}+R_{\beta,t}}^2-\E\norm{e_{\beta,t}}^2\right|
 \le2\rho\sqrt{\E\norm{e_{\beta,t}}^2}+\rho^2.
\]
The residual represents nonlinear normalization and deviations from the local statistical model. The inequality quantifies how a stated residual bound transfers the model risk to the actual direction-error risk. Applying it along an adaptive trajectory requires a bound on the residual at the corresponding updates. The workload measurements audit the practical proxies and selected configurations.

\subsection{Comparing memories at finite measurement resolution}\label{app:resolution}
Proposition~\ref{prop:resolution} quantifies the allowance introduced by estimating a slope from noisy endpoints. This subsection measures its effect on the memory decision under the same local model.
Fix a finite candidate set $\mathcal B\subset[0,1)$ and the update $t$. Define $R_t(\beta)=\E\norm{Pe_{\beta,t}}^2$, and let $\beta_R$ minimize $R_t$ and $\beta_U$ minimize $U_{L,t}$ on $\mathcal B$. The coefficients in both functions are the population and expected-measurement coefficients specified in the proposition. Then
\[
 0\le R_t(\beta_U)-R_t(\beta_R)
 \le\frac{C}{L^2}\bigl(a_{\beta_R,t}^2-a_{\beta_U,t}^2\bigr)
 \le\frac{C}{L^2}a_{\beta_R,t}^2.
\]
To prove this, start with $U_{L,t}(\beta_U)\le U_{L,t}(\beta_R)$ and substitute~\eqref{eq:resolution} on both sides. Rearranging gives the middle inequality; optimality of $\beta_R$ gives the lower bound. Since mean ages are nonnegative, discarding the subtracted age term gives the last bound. The result compares two memory decisions through their local tracking risk and makes the cost of measurement resolution explicit.

Finite coefficient accuracy can be incorporated in the same comparison. Suppose a measured score $\widehat U_{L,t}$ satisfies $|\widehat U_{L,t}(\beta)-U_{L,t}(\beta)|\le\delta$ simultaneously for every $\beta\in\mathcal B$, and let $\widehat\beta$ minimize this measured score. Inserting the two uniform error bounds into $\widehat U_{L,t}(\widehat\beta)\le\widehat U_{L,t}(\beta_U)$ yields $U_{L,t}(\widehat\beta)\le U_{L,t}(\beta_U)+2\delta$. Hence
\[
 R_t(\widehat\beta)-R_t(\beta_R)
 \le\frac{C}{L^2}a_{\beta_R,t}^2+2\delta.
\]
Here $\delta$ is an assumed uniform coefficient-score accuracy, and the bound is conditional on it. For example, bounds $|\widehat V_{\rm emp}-V_P|\le e_V$ and $|\widehat D_{\rm emp}-D_L|\le e_D$ give $\delta=\max_{\beta\in\mathcal B}(e_Vc_{\beta,t}+e_Da_{\beta,t}^2)$. For a cubic-and-rounded candidate $b$, define its measured optimization discrepancy as $\widehat\eta=\widehat U_{L,t}(b)-\min_{\beta\in\mathcal B}\widehat U_{L,t}(\beta)$. The same argument gives the bound $R_t(b)-R_t(\beta_R)\le(C/L^2)a_{\beta_R,t}^2+2\delta+\widehat\eta$: insert this discrepancy before applying the two uniform error bounds. Thus coefficient accuracy and optimization accuracy have separate roles. The exact finite-history audit finds $\widehat\eta=0$ for all eleven observed choices at $t=200$. Timing and decision-margin audits describe their additional empirical sensitivity.

\subsection{A local connection to the loss}
Tracking error matters because it perturbs a reference update. The following one-step inequality separates a systematic directional shift from the curvature cost of fluctuations.
If $F$ has an $L$-Lipschitz gradient, $\theta,r$ are fixed, and $u=r+e$ has finite second moment, the smoothness inequality at $\theta-\eta r$ gives
\begin{align*}
 \E F(\theta-\eta u)
 &\le F(\theta-\eta r)-\eta\langle\nabla F(\theta-\eta r),\E e\rangle+\frac{L\eta^2}{2}\E\norm e^2\\
 &\le F(\theta-\eta r)+\eta\norm{\nabla F(\theta-\eta r)}\norm{\E e}+\frac{L\eta^2}{2}\E\norm e^2.
\end{align*}
The term involving $\E e$ measures a systematic shift from the reference direction; the term involving $\E\norm e^2$ is the curvature penalty. This one-step inequality connects direction-error moments to the loss under the stated smoothness assumption. Its parameter-space norm makes a norm-transfer bound necessary when the error is measured through a sketch. The empirical comparisons evaluate the selected configurations through their full-training validation losses.

\subsection{Numerical checks of the identities}
These checks test the algebra and the specified innovation model, independently of the saved neural-network outcomes.
With NumPy seed 20261005, 1,000 random-vector cases verify~\eqref{eq:pool} with maximum absolute residual $7.11\times10^{-15}$, and 1,000 positive-coefficient cases verify~\eqref{eq:sensitivity} with residual $1.33\times10^{-15}$. An affine-model check uses 30,000 independent histories of length 128, $\beta$ 0.95, innovation standard deviation 0.3, and slope 0.002:
\[
 e=\sum_{k=0}^{127}w_{0.95,128}(k)\xi_k-0.002\sum_{k=0}^{127}kw_{0.95,128}(k),\qquad \xi_k\sim\mathcal N(0,0.3^2).
\]
Its measured mean squared error is 0.00373524, versus exact 0.00373090; the Monte Carlo standard error is $2.84\times10^{-5}$. These checks test the statistical formulas, separately from network training.

\section{Measurement reproducibility}\label{app:measurement}
To reproduce a selection, one needs four direction measurements and the fixed geometry used to compare them. The following specification covers gradient processing, state preservation, projection, and scalar calculation in that order.
\subsection{Gradient samples and probe state}
The pilot seed is 1 and shared $\beta$ 0.95. Models remain in training mode during probes, including dropout and augmentation. For accumulation count $a$, each effective-batch gradient averages $a$ physical-batch losses. Norm clipping follows accumulation. Coupled Adam decay is added after clipping; AdamW's decoupled decay belongs to the parameter update. Sixteen fresh effective batches are drawn per required checkpoint with parameters held fixed. Partition their ordered samples into four groups of four; each group computes its own first and second moments before normalization. The pooled direction uses all sixteen samples. Thus the dispersion calculation compares four group directions, each formed from four samples, with one pooled direction.

Python, NumPy, CPU/CUDA random states, module training flags, and mutable buffers are restored after probing, preserving the pilot's training state. The dedicated probe iterator and generator advance through the sampled examples. Gradients are cleared afterward. Dense archives contain checkpoints through update 500 and additional probe budgets; the rule replays checkpoints 100, 150, 180, and 200 with sixteen samples each. Matching the original probe examples in a sparse collection requires reproducing the dense stream's advancement. The supplied measured sketches provide exact reconstruction of the reported decisions. The 64-gradient operation count describes a prospective compact four-checkpoint deployment; the experiments reconstruct its decision from the dense archive's observations.

Pilot vision loaders use zero worker processes and a separate probe generator seeded by training seed plus 10000019. GPT pilot training and probe generators use seed plus 101 and seed plus 10000101, respectively. Llama's training shuffle uses seed 42 and the probe shuffle uses seed plus 10000003.

\subsection{The exact fixed sketch}
The sketch makes measurements storable while giving each parameter tensor comparable weight. It defines a tensor-balanced measurement norm. For example, multiplying one tensor's dimension by four reduces each of its coordinate weights by two. The same fixed linear map applies to the theorem's innovations and slope, giving the tracking identity in this measurement geometry. Transferring a sketch-space error bound to the original parameter-space norm requires an accompanying norm-transfer bound.
Use dimension 1,024 and seed 12,345. For trainable tensor $p$ with $n_p$ entries, assign coordinate $i$ to bucket $b_{p,i}$ and sign $s_{p,i}\in\{-1,1\}$. The projection is
\[
 (Pr)_b=\sum_p\frac1{\sqrt{n_p}}\sum_{i:b_{p,i}=b}s_{p,i}r_{p,i}.
\]
To obtain each tensor's bucket/sign generator seed, join the decimal projection seed, parameter name, and suffix \texttt{idx} or \texttt{sgn} with \texttt{|}. Hash its UTF-8 bytes with BLAKE2b, eight-byte digest; interpret little-endian unsigned and retain the low 31 bits. Seed independent Torch CPU generators for bucket and sign draws. Buckets are uniform integers 0--1023; signs are generated by uniform integers 0 or 1 mapped to $-1$ or $+1$. Flatten tensors and accumulate in float64. Names are the model's named trainable parameters. Global sketches sum tensor sketches and supply the selection inputs. The archive also contains layerwise 32-coordinate diagnostic sketches.

\subsection{Statistical basis of the coefficient estimates}\label{app:estimatorproofs}
The main estimators use two standard operations: variance estimation from independent groups and linear regression on the lag responses. This subsection supplies the calculations and states precisely where their unbiasedness applies.

\paragraph{Grouped variability.}
At a fixed checkpoint, let $Z_i=P\zeta_i$ be independent, centered linearized direction innovations with $\E\norm{Z_i}^2=V_P$. Each of $G$ disjoint groups contains $m/G$ observations. Its innovation average is $Z^{(j)}=(G/m)\sum_{i\text{ in group }j}Z_i$, and their common mean is $\bar Z=G^{-1}\sum_jZ^{(j)}$. Independence and centering give
\[
 \E\norm{Z^{(j)}}^2=\frac{G}{m}V_P,\qquad
 \E\norm{\bar Z}^2=\frac{1}{m}V_P.
\]
The identity $G^{-1}\sum_j\norm{Z^{(j)}-\bar Z}^2=G^{-1}\sum_j\norm{Z^{(j)}}^2-\norm{\bar Z}^2$ therefore yields
\[
 \E\left[\frac1G\sum_j\norm{Z^{(j)}-\bar Z}^2\right]
 =\frac{G-1}{m}V_P.
\]
In the linearized model, subtracting the common population direction from pooled and group estimates leaves these innovation averages. Equation~\eqref{eq:empV} thus estimates $V_P$ without bias under this model. For nonlinear normalized moment estimates, Proposition~\ref{prop:pool} still holds exactly, but the group vectors also contain higher-order terms. Its pooling displacement describes one observable contribution; the residual analysis in Appendix~\ref{app:scope} addresses the additional approximation error.

\paragraph{Temporal regression.}
Condition on the pilot parameter trajectory in the observation model of Section~\ref{sec:rule}. The lags and projected population directions are then fixed. Write $Y=(y_{n_1},\ldots,y_{n_K})^\top$ and let $X$ be the $K$-by-two matrix whose row $\ell$ is $(1,n_\ell^{-2})$. Equation~\eqref{eq:endpoint} gives $Y=X(D_P,C)^\top+\nu$, where $\nu$ is the regression residual vector and $\E\nu=0$. Distinct lags make the two columns independent. The ordinary least-squares coefficients satisfy
\[
 \begin{pmatrix}I_{\rm fit}\\C_{\rm fit}\end{pmatrix}
 =(X^\top X)^{-1}X^\top Y
 =\begin{pmatrix}D_P\\C\end{pmatrix}
 +(X^\top X)^{-1}X^\top\nu.
\]
Taking expectations proves unbiasedness of the unconstrained coefficients. The errors in $Y$ may be correlated and have different variances; these features affect uncertainty and efficiency, while the expectation calculation only uses their zero means. The implementation's least-squares choice matches the affine mean relation and requires no distributional likelihood assumption. Componentwise clipping subsequently enforces the nonnegative score coefficients and can alter their sampling expectations, particularly near zero.

\subsection{Coefficient calculation and conversion}
The calculation uses fresh probe moments: pooled directions for checkpoint changes and individual group directions for fluctuations at a fixed checkpoint.
For $m=16,G=4$, $\widehat V_{\rm emp}=(4/3)\sum_{j=1}^4\norm{z_{200}^{(j)}-z_{200}}^2$. Each group normalizes moments from four effective-batch gradients; the pooled direction uses moments from all sixteen. The archived collector also uses a square-root floor $\tau=10^{-30}$ in auxiliary derivative calculations. Its guarded direction is numerically identical to~\eqref{eq:probe} in the collector's float64 arithmetic with $\eps=10^{-8}$: the floor changes a denominator only when $\sqrt{\widehat q}<\tau$, and both sums with epsilon then round to epsilon. For exact arithmetic the coordinatewise difference is at most $\tau^2/(4\eps^2)=2.5\times10^{-45}$, since $|\widehat\mu|\le\sqrt{\widehat q}$. Thus the direction estimator can be written in the standard Adam form without changing the archived selection inputs.

Build a three-by-two least-squares design with rows $(1,n^{-2})$ for $n=20,50,100$, and response $y_n$. Solve for $(I,C)$, apply componentwise clipping $\widehat I=\max(0,I)$ and $\widehat C=\max(0,C)$, and evaluate at lag 100. The algorithm therefore consists of one ordinary least-squares fit followed by clipping. Raw checkpoint sketches supply the lag observations directly. The resulting $\widehat D_{\rm emp}$ is the fitted nonnegative change on the stated observation scale, including its finite-lag contribution.

For positive coefficients, calculate the cubic memory in~\eqref{eq:rule}. The admissible interval is $H_{\min}=1/(1-0.68377)$ through $H_{\max}=1/(1-0.999)$; restrict memory to this interval and convert it to $\beta$. Equivalently, this minimizes the operational leading score over the continuous interval before nearest-$\beta$ rounding. If $\widehat D_{\rm emp}=0<\widehat V_{\rm emp}$, the score decreases with memory, so use $H_{\max}$. If $\widehat V_{\rm emp}=0<\widehat D_{\rm emp}$, use $H_{\min}$. If both vanish, all scores tie and the smaller $\beta$ resolves the tie. These cases follow from minimizing the stated objective. All main measurements have positive coefficients. Rounding is in $\beta$ distance, with ties toward the smaller grid value.

Appendix~\ref{app:estimatorproofs} derives the variance correction; Proposition~\ref{prop:pool} separates group scatter and nonlinear pooling displacement.

\section{Training reproducibility}\label{app:reproduction}
The selected configurations are evaluated using fixed-memory training runs. This appendix specifies the common protocol, then the architecture, data processing, and validation for each family. Offline reconstruction from the supplied measurements is described separately from regenerating the original training streams.
\subsection{Shared settings and schedules}
All sweep runs use equal $\beta$ values and epsilon $10^{-8}$ outside the root. For each workload, candidates share batch size, architecture, data, learning-rate recipe, and evaluation schedule. The main comparison uses seed 1. Initialization seeds Python, NumPy, Torch CPU, and all CUDA devices; cuDNN uses deterministic mode with benchmarking set to \texttt{False}. Models use random initialization, except T5-small, which uses its pretrained checkpoint. Pilots use CUDA and bfloat16 autocast; fused optimizers are used when available.

Full-training budgets are 10,000 updates for language models and T5; 50 epochs for Food-101 and both ViTs; 80 epochs for ImageNet100; 100 epochs for Cars; 60 epochs for Caltech-256. Vision full-training schedule horizons equal epochs times training batches per epoch. Their rates warm up linearly, then decay as
\[
 \eta_s=\eta_{\min}+\frac{\eta_{\max}-\eta_{\min}}2\left[1+\cos\left(\pi\frac{s-W}{T-W}\right)\right],\qquad s>W,
\]
with $\eta_s=\eta_{\max}s/W$ during warmup. ResNet warmup is 5\% of the horizon; ViT uses 1,000 updates; EfficientNet and Swin use the larger of 100 updates and 5\% of the horizon. Language warmups are 1,000 for Llama and T5, and 2,000 for GPT-2-style models.

Pilot schedule horizons are explicitly configured as listed in Table~\ref{tab:pilot}, including ImageNet100 and Caltech-256. Pilot warmup lengths are 1,000 for Llama and T5; 2,000 for GPT-2-style models; 1,477 for Food; 4,060 for ImageNet100; 1,000 for ViTs; 640 for Cars; and 1,128 for Swin. Thus all four compact measurements occur during warmup.
\begin{table}[ht]\centering
\caption{Pilot schedule and optimizer settings. Steps is the configured schedule horizon; batch is physical batch size. Accumulation defines the effective batch. Clip zero denotes an unmodified gradient norm.}
\label{tab:pilot}\begin{tabular*}{\textwidth}{@{\extracolsep{\fill}}llrrrrrr@{}}
\toprule Network & Dataset & Steps & Batch & Accum. & LR max & Decay & Clip\\\midrule
EfficientNet-B0 & Cars & 12800 & 64 & 1 & 0.0008 & 5e-05 & 1\\
Llama60M & C4 & 10000 & 64 & 8 & 0.001 & 1e-05 & 0\\
Llama60M & SlimPajama & 10000 & 64 & 8 & 0.001 & 1e-05 & 0\\
NanoGPT & OpenWebText & 10000 & 8 & 128 & 0.0006 & 0.01 & 1\\
NanoGPT & WikiText-103 & 10000 & 8 & 128 & 0.0006 & 0.01 & 1\\
ResNet50 & Food-101 & 29550 & 128 & 1 & 0.0003 & 0.01 & 0\\
ResNet50 & ImageNet100 & 81200 & 128 & 1 & 0.0003 & 0.01 & 0\\
Swin-T & Caltech-256 & 22560 & 64 & 1 & 0.0004 & 0.05 & 1\\
T5-small & BookCorpus & 10000 & 8 & 8 & 0.0001 & 0.01 & 1\\
ViT-B/16 & CIFAR-100 & 19500 & 128 & 1 & 0.0005 & 0.1 & 1\\
ViT-B/16 & TinyImageNet & 39050 & 128 & 1 & 0.0005 & 0.1 & 1\\
\bottomrule\end{tabular*}

\end{table}

AdamW excludes one-dimensional parameters, biases, and normalization parameters from decay according to the recipe: batch-normalization names for ResNet; normalization names for GPT, ViT, Swin, and T5. Llama uses coupled Adam decay on all parameters; EfficientNet also uses coupled Adam decay. Clipping is on the full accumulated gradient norm before the optimizer update.

\subsection{Causal language modeling}
\paragraph{GPT-2-style models.}
Both have six decoder blocks, width 384, six attention heads, feedforward width 1,536, vocabulary 50,304, and context 256. They use learned position embeddings, tied input/output embeddings, layer normalization epsilon $10^{-5}$, GPT-2 approximate GELU, initialization standard deviation 0.02, and dropout 0.2 in embeddings, attention, and residuals. Each effective batch averages 128 physical batches of eight sequences (262,144 input tokens). AdamW uses maximum rate $6\times10^{-4}$, minimum $6\times10^{-6}$, decay 0.01, and clip 1.

WikiText-103 uses its raw training/validation splits and GPT-2 tokenization \citep{merity2017pointer}. OpenWebText uses separately packed training/validation streams \citep{gokaslan2019openwebtext}. Batches sample uniform contiguous start positions, giving length-256 inputs and one-token-shifted targets. Evaluation every 500 updates averages next-token cross-entropy over 50 validation batches. Its validation generator uses seed 123456789; train-set evaluation uses 987654321. The packed OpenWebText streams define the actual split; their corpus revision and packing manifest are required to match the exact archived samples.

\paragraph{Llama-style models.}
Both have eight decoder blocks, width 512, feedforward width 1,376, eight attention heads, SiLU-gated feedforward layers, rotary positional encoding, RMS normalization epsilon $10^{-6}$, and initialization standard deviation 0.02 \citep{touvron2023llama}. The configured vocabulary begins at 32,000 and is expanded for the T5-base tokenizer when required. Sequences have length 256; eight physical batches of 64 sequences form each effective batch (131,072 token positions). Adam uses maximum rate $10^{-3}$, minimum $10^{-4}$, coupled decay $10^{-5}$, and gradients with their original norms.

C4 uses the English configuration \citep{raffel2020exploring}; the second task uses the cached DKYoon/SlimPajama-6B subset \citep{soboleva2023slimpajama}. Texts are shuffled with seed 42, tokenized with T5-base, truncated/padded to 256, and evaluated with valid-token masking. Validation uses batch size 64 and a ten-million-valid-token target. Evaluation is recorded at update 1 and every 1,000 updates, aggregating losses by valid target count.

\subsection{Image classification models}
\paragraph{ResNet50.}
Bottleneck stages have depths $(3,4,6,3)$ and output widths $(256,512,1024,2048)$, using ReLU, batch normalization, a $7\times7$ stride-2 stem and max pooling, and global average pooling \citep{he2016deep}. Dropout 0.2 precedes the linear 101-class Food or 100-class ImageNet head. AdamW uses maximum rate $3\times10^{-4}$, minimum $3\times10^{-5}$, decay 0.01, gradients with their original norms, and batch 128.

\paragraph{ViT-B/16.}
Inputs are $224\times224$ RGB images with $16\times16$ patches. The encoder has 12 blocks, width 768, 12 heads, feedforward width 3,072, GELU, layer-normalization epsilon $10^{-6}$, and the standard torchvision zero encoder/attention-dropout defaults \citep{dosovitskiy2021image}. Heads have 100 CIFAR or 200 TinyImageNet outputs. AdamW uses maximum rate $5\times10^{-4}$, minimum $5\times10^{-6}$, decay 0.1, clip 1, and batch 128.

\paragraph{Swin-T.}
Patch size is 4, window size 7, base width 96, stage depths $(2,2,6,2)$, head counts $(3,6,12,24)$, and feedforward expansion 4 \citep{liu2021swin}. Shifted-window attention, GELU, layer normalization, and stochastic depth 0.2 are used, with a 256-class head. AdamW uses maximum rate $4\times10^{-4}$, minimum $4\times10^{-6}$, decay 0.05, clip 1, and batch 64.

\paragraph{EfficientNet-B0.}
The 32-channel stem is followed by seven MBConv stages. Their (expansion, kernel, first stride, output channels, repetitions) are $(1,3,1,16,1)$, $(6,3,2,24,2)$, $(6,5,2,40,2)$, $(6,3,2,80,3)$, $(6,5,1,112,3)$, $(6,5,2,192,4)$, and $(6,3,1,320,1)$ \citep{tan2019efficientnet}. Squeeze-and-excitation, SiLU, stochastic depth 0.2, a final width of 1,280, global pooling, dropout 0.2, and a 196-class head complete the model. Adam uses maximum rate $8\times10^{-4}$, minimum $10^{-5}$, coupled decay $5\times10^{-5}$, clip 1, and batch 64.

\subsection{Vision datasets, preprocessing, and evaluation}
Food-101 uses 75,750 training and 25,250 validation images \citep{bossard2014food}. Stanford Cars uses 8,144/8,041 images, with its test partition used for validation \citep{krause2013collecting}. CIFAR-100 uses 50,000/10,000 images \citep{krizhevsky2009learning}. TinyImageNet uses 100,000 training and 10,000 validation images across 200 classes. ImageNet100 is a 100-class subset \citep{deng2009imagenet}; the pilot loads 126,689 training images and remaps labels consistently across four training partitions and validation. Exact subset membership is determined by its synset manifest.

Caltech-256 excludes clutter \citep{griffin2007caltech256}. Sort category names and image paths lexicographically, create one Python random generator with seed 123, visit classes in order and shuffle each class's paths, then take the first 60 images for training and the remainder for validation. The generator continues across classes. This gives 15,360 training images.

All inputs are converted to RGB. ResNet uses random resized crop 224, area scale $(0.6,1)$, aspect ratio $(3/4,4/3)$, horizontal flip probability 0.5, and brightness/contrast/saturation/hue jitter $(0.2,0.2,0.2,0.1)$. ViT and Swin use crop area scale $(0.5,1)$, the same ratio/flip, RandAugment with two operations and magnitude 12, and random erasing probability 0.25, area $(0.02,0.2)$, ratio $(0.3,3.3)$, random fill.

EfficientNet uses crop area $(0.7,1)$, ratio $(0.85,1.15)$, flip 0.5, color jitter $(0.2,0.2,0.15,0.03)$, RandAugment with one operation and magnitude 7, and erasing probability 0.1, area $(0.02,0.12)$, ratio $(0.5,2)$, random fill. Erasing follows tensor conversion and normalization. Evaluation resizes the shorter side to 256 and center crops to 224.

Normalization uses means $(0.485,0.456,0.406)$ and standard deviations $(0.229,0.224,0.225)$; CIFAR uses $(0.5071,0.4867,0.4408)$ and $(0.2675,0.2565,0.2761)$. Training uses cross-entropy with label smoothing 0.1, except EfficientNet with 0.05. \emph{Validation uses plain cross-entropy with label smoothing set to zero.} Validation averages batch-mean losses, weighting the final partial batch equally at batch level. Evaluation occurs every epoch. Training shuffles and drops partial batches except EfficientNet; validation preserves dataset order and includes every batch.

\subsection{T5 denoising}
T5-small begins from the pretrained checkpoint and tokenizer \citep{raffel2020exploring}. Both encoder and decoder have six blocks, width 512, feedforward width 2,048, eight heads with key/value width 64, ReLU, normalization epsilon $10^{-6}$, and dropout 0.1. BookCorpus \citep{zhu2015aligning} uses span corruption with maximum input length 256, noise density 0.15, mean span length 3, and sentinel tokens. Padded targets are masked with $-100$.

Corruption is seeded by training seed plus example index; validation adds 10,000,000. Full training reserves the first 20,000 corpus records for validation and uses the remaining 73,984,228 for training. Here a record is one indexed text row. The pilot reserves the first 1,000 records. These are distinct pilot/full partitions. Eight physical batches of eight form each effective batch. AdamW uses maximum rate $10^{-4}$, minimum $10^{-5}$, decay 0.01, and clip 1. The 10,000-update training evaluates every 1,000 updates with batch size eight and valid-target-token weighting. The ten-workload comparison excluding T5 gives mean gaps 0.6850\% for fixed $\beta$ and 0.4060\% for the rule.

\subsection{Metrics and implementation provenance}
Take each run's minimum recorded validation loss. For the primary table, the smallest of the eleven seed-1 grid losses defines~\eqref{eq:gap}'s reference. For the common-seed comparison, first take each run's minimum and then average those minima over identical seed sets.

The implementation uses PyTorch, torchvision, Transformers, NumPy, and dataset-loading utilities. Training was executed on a cluster equipped with eight NVIDIA A100-SXM4 GPUs with 80\,GB memory each; this identifies available cluster hardware. The retained execution records identify CUDA and pilot bfloat16 autocast. Exact regeneration of the original training streams additionally requires pinned package versions, the ImageNet100 synset list, packed OpenWebText split manifest, and corpus/checkpoint revisions. The training recipes and matched-seed evaluation are specified here; those additional identifiers determine the exact environment and sample streams.

\subsection{Reconstructing the selection evaluation}
For each workload, use the pooled and group directions at its four measurement checkpoints. Calculate $\widehat V_{\rm emp}$ by~\eqref{eq:empV}, fit the three squared secants by~\eqref{eq:lagfit}, calculate $\widehat D_{\rm emp}$ by~\eqref{eq:empD}, and convert their ratio to a grid $\beta$ by Section~\ref{sec:measurements}. Then evaluate that $\beta$'s recorded validation loss and compute its relative gap by~\eqref{eq:gap}. Take the arithmetic mean of the eleven gaps, their maximum, and the mean of their three largest values. Applying the identical aggregation to a fixed $\beta$ gives the comparator. The resulting mean gaps are 0.6235665\% for $\beta$ 0.94377 and 0.3694690\% for the rule. Timing and probe-budget analyses repeat these steps with their stated measurement designs, holding the candidate trainings fixed.

\paragraph{Existing corpus content.}
The language tasks use existing text corpora, which can contain offensive text and personally identifying information inherited from their sources. The study uses these benchmark assets for aggregate training and validation measurements. Released measurements consist of gradient-direction sketches and aggregated validation losses. Use of the original datasets and checkpoints remains subject to their providers' access conditions and licenses.

\subsection{Existing assets and access conditions}\label{app:assets}
Software and checkpoint licenses and dataset-provider terms govern access to the original assets. Table~\ref{tab:assets} identifies their providers and documented license information. Released measurements consist of gradient-direction sketches and aggregated losses.
\begin{table}[ht]\centering
\caption{Software and model license records, and dataset access sources. The providers' records govern access and reuse of the original assets.}
\label{tab:assets}
\begin{tabular*}{\textwidth}{@{\extracolsep{\fill}}p{0.22\textwidth}p{0.69\textwidth}@{}}
\toprule Asset & License information or provider record\\\midrule
PyTorch / torchvision & BSD-style / BSD-3-Clause: \href{https://github.com/pytorch/pytorch/blob/main/LICENSE}{PyTorch license} and \href{https://github.com/pytorch/vision/blob/main/LICENSE}{torchvision license}.\\
NumPy / Transformers & BSD-3-Clause / Apache-2.0: \href{https://github.com/numpy/numpy/blob/main/LICENSE.txt}{NumPy license} and \href{https://github.com/huggingface/transformers/blob/main/LICENSE}{Transformers license}.\\
T5-small checkpoint & Apache-2.0, as stated in the \href{https://huggingface.co/google-t5/t5-small}{provider model card}. The other evaluated models start from random weights.\\
C4 / WikiText-103 & C4 card: \href{https://huggingface.co/datasets/allenai/c4}{ODC-By}. WikiText card: \href{https://huggingface.co/datasets/Salesforce/wikitext}{CC-BY-SA-3.0 and GFDL}.\\
SlimPajama-6B & The \href{https://huggingface.co/datasets/DKYoon/SlimPajama-6B}{subset card} points to the original SlimPajama record for source and license information. This original provider record supplies the access provenance used here.\\
OpenWebText / BookCorpus & Provider records: \href{https://huggingface.co/datasets/Skylion007/openwebtext}{OpenWebText} and \href{https://huggingface.co/datasets/bookcorpus/bookcorpus}{BookCorpus}. Rights in underlying web text and books remain with their respective sources.\\
ImageNet / TinyImageNet & \href{https://www.image-net.org/download.php}{ImageNet access terms} specify noncommercial research and educational use. The \href{https://cs231n.stanford.edu/tiny-imagenet-200.zip}{TinyImageNet release} supplies its distribution provenance; consult the provider's terms for reuse.\\
Caltech-256 & \href{https://data.caltech.edu/records/nyy15-4j048}{CaltechDATA provider record}, including its rights metadata.\\
Food-101 / Cars / CIFAR & Original access sources: \href{https://data.vision.ee.ethz.ch/cvl/datasets_extra/food-101/}{Food-101}, \href{https://ai.stanford.edu/~jkrause/cars/car_dataset.html}{Stanford Cars}, and \href{https://www.cs.toronto.edu/~kriz/cifar.html}{CIFAR-100}. These provider records supply the dataset-access provenance used here.\\
\bottomrule
\end{tabular*}
\end{table}

\section{Ablations and additional measurements}\label{app:ablations}
The following analyses address separate questions: whether a nearby decision time retains the benefit, how many probe batches are needed, how measured coefficients determine $\beta$, and how the selected configurations behave over common training seeds. Each analysis reuses saved candidate losses. Measurements select the configuration; the loss lookup evaluates it afterward.
\subsection{Decision timing and probe budget}
First hold the probe budget at sixteen and change the decision update. Table~\ref{tab:timingbudget} reports updates 100, 150, 200, and 250. For a decision at update $s$, measure at $s-100,s-50,s-20,s$ and use exactly the same fit, cubic rule, and rounding. For instance, a decision at update 150 uses observations at 50, 100, 130, and 150. Shifting this window changes the observed dynamics while preserving the measurement scales and training recipes.

\begin{table}[ht]\centering
\caption{Decision-time sensitivity with sixteen probes per checkpoint and matched seed 1. Gaps are percentages. The near-0.95 comparator, $\beta$ 0.94377, has mean 0.6236, maximum 2.4853, and CVaR 2.0409.}
\label{tab:timingbudget}
\begin{tabular*}{\textwidth}{@{\extracolsep{\fill}}rrrr@{}}
\toprule Decision update & Mean gap & Maximum gap & CVaR$_{25}$\\\midrule
100 & 0.6345 & 2.1862 & 1.5271\\
150 & 0.3695 & 1.3339 & 1.1359\\
200 & 0.3695 & 1.3339 & 1.1359\\
250 & 0.3695 & 1.3339 & 1.1359\\
\bottomrule\end{tabular*}

\end{table}

Updates 150, 200, and 250 give the same aggregate result. Update 100 has a slightly larger mean gap and a smaller maximum and CVaR than the near-0.95 reference. These results locate a useful early observation region. Appendix~\ref{app:timingfull} extends the analysis to every integer update from 100 to 500 on the same pilot trajectories and fixed candidate trainings.

\subsection{Probe allocation and budget sensitivity}\label{app:budgetdesign}
The common sixteen-probe design assigns four gradients to each of four groups. Averaging the gradients of each group supplies its first and second moment estimates before normalization; comparing groups supplies a dispersion measurement. The pooled direction uses all sixteen gradients. Thus the budget supports two statistical roles, rather than sixteen independent votes for a $\beta$. Powers of two provide equal groups throughout the measured allocation $m\in\{2,4,8,16\}$, with $G=\min(m,4)$.

The averaging role has a simple theoretical scale. If independent linear direction observations have covariance $\Sigma$, their $m$-sample mean has covariance $\Sigma/m$: expand the covariance of the sum and use independence to eliminate cross terms. Under independent, equal-variance endpoints, the squared secant has an expected noise contribution $2\operatorname{tr}(\Sigma)/(mn^2)$ at lag $n$. This explains why the observation budget affects the finite-resolution term in the temporal fit. For the implemented nonlinear normalization, Proposition~\ref{prop:pool} separately accounts for pooling order. The choice $m=16,G=4$ is a compact common allocation; these sampling identities explain its roles and scaling, while the measured budget audit evaluates its selections.

Table~\ref{tab:budget} holds update 200, the coefficient formula, and the rounding convention fixed, and reports every examined budget. The variability multiplier remains $m/(G-1)$ throughout. Four and eight probes produce the same discrete choices, while sixteen changes the selections for EfficientNet, both ResNets, and T5. The table therefore exposes how increased averaging and changed endpoint resolution translate into the practical memory decision.
\begin{table}[ht]\centering
\caption{Probe-budget sensitivity at update 200. Relative gaps are percentages; the fixed comparator has mean 0.6236\%, maximum 2.4853\%, and CVaR$_{25}$ 2.0409\%. Bold marks the smallest entry in each metric column.}
\label{tab:budget}\begin{tabular*}{\textwidth}{@{\extracolsep{\fill}}rrrrr@{}}
\toprule Effective batches & Groups & Mean gap & Maximum gap & CVaR$_{25}$\\\midrule
2 & 2 & 0.9185 & 2.5913 & 2.4600\\
4 & 4 & 0.6829 & 2.4853 & 2.0409\\
8 & 4 & 0.6829 & 2.4853 & 2.0409\\
16 & 4 & \textbf{0.3695} & \textbf{1.3339} & \textbf{1.1359}\\
\bottomrule\end{tabular*}
\end{table}
The common allocation is used unchanged across tasks. Workload-specific information enters through measured variability and temporal change. The coefficient sensitivity and rounding margins quantify how these inputs determine the continuous memory and the final grid choice.

\paragraph{Separation of measurement and evaluation.}
For each protocol, gradient measurements determine the coefficients and the $\beta$ before the candidate-loss lookup. Candidate losses then evaluate the resulting configuration. The same equations and rounding apply to every workload. Protocol comparisons are exploratory sensitivity analyses on this benchmark suite, describing how the decisions respond to the implementation choices.

\subsection{Comparison with every fixed grid value}\label{app:fixedgrid}
Table~\ref{tab:fixedgrid} evaluates all eleven fixed $\beta$ values using the same seed-1 losses and aggregation. Each constant applies to the entire workload suite. The rule has smaller mean, maximum, and CVaR than every constant, reducing mean gap by 32.3\% against the retrospectively best fixed value, 0.96838. This evaluates the same gradient-based decisions under a stronger comparator than the near-0.95 reference.
\begin{table}[ht]\centering
\caption{All fixed shared memories on the evaluated grid, and the tracking rule. Gaps are percentages. Bold marks the smallest value in each metric column.}
\label{tab:fixedgrid}\begin{tabular*}{\textwidth}{@{\extracolsep{\fill}}lrrr@{}}
\toprule Fixed shared $\beta$ or method & Mean gap & Maximum gap & CVaR$_{25}$\\\midrule
0.68377 & 4.8152 & 11.6651 & 9.9106\\
0.82217 & 2.0501 & 6.4732 & 4.5653\\
0.90000 & 1.4292 & 3.8770 & 3.3543\\
0.94377 & 0.6236 & 2.4853 & 2.0409\\
0.96838 & 0.5461 & 2.1862 & 1.4304\\
0.98222 & 1.8128 & 7.2826 & 4.5072\\
0.99000 & 3.1592 & 9.4700 & 8.2164\\
0.99438 & 7.9687 & 24.6070 & 18.4218\\
0.99684 & 13.6732 & 34.5576 & 28.0746\\
0.99822 & 27.6722 & 63.7899 & 55.7021\\
0.99900 & 42.7583 & 81.3312 & 73.2353\\
Tracking rule & \textbf{0.3695} & \textbf{1.3339} & \textbf{1.1359}\\
\bottomrule\end{tabular*}

\end{table}

\subsection{Task-level calibration and workload composition}\label{app:compromise}
The two strongest fixed choices expose different compromises. Shared $\beta$ 0.94377 minimizes the four causal-language-model losses and the vision-transformer family mean in the recorded grid. Shared $\beta$ 0.96838 minimizes the ResNet family mean and the whole-suite fixed mean. The measured rule combines these task-specific choices with its shorter causal-model memory.

For a comparator $b$, define gap benefit as $\operatorname{gap}_e(b)-\operatorname{gap}_e(\widehat\beta_e)$, in percentage points, and direct loss improvement as $100[L_e(b)-L_e(\widehat\beta_e)]/L_e(b)$. Positive values favor the rule. Table~\ref{tab:attribution} separates those metrics for both fixed comparators. Relative gap reduction is the reduction of distance to the grid reference; direct loss improvement gives the scale in the actual objective.
\begin{table}[ht]\centering
\caption{Signed task-level benefits against the near-0.95 fixed value and the best global fixed value. Gap benefits are percentage points; direct loss improvements are percentages. Positive entries favor the rule.}
\label{tab:attribution}\begin{tabular*}{\textwidth}{@{\extracolsep{\fill}}llrrrrrr@{}}
\toprule & & & & \multicolumn{2}{c}{Gap benefit (pp)} & \multicolumn{2}{c}{Loss improvement (\%)}\\ \cmidrule(lr){5-6}\cmidrule(lr){7-8} Network & Dataset & $\widehat\beta$ & Rule gap & .94377 & .96838 & .94377 & .96838\\\midrule
EfficientNet-B0 & Cars & 0.96838 & 1.0124 & -0.2853 & +0.0000 & -0.2833 & +0.0000\\
Llama60M & C4 & 0.90000 & 0.2582 & -0.2582 & -0.0700 & -0.2582 & -0.0699\\
Llama60M & SlimPajama & 0.90000 & 0.1175 & -0.1175 & -0.0924 & -0.1175 & -0.0924\\
NanoGPT & OpenWebText & 0.90000 & 0.1157 & -0.1157 & -0.1137 & -0.1157 & -0.1137\\
NanoGPT & WikiText & 0.90000 & 0.1613 & -0.1613 & -0.0800 & -0.1613 & -0.0800\\
ResNet50 & Food-101 & 0.96838 & 1.0613 & +1.4240 & +0.0000 & +1.3895 & +0.0000\\
ResNet50 & ImageNet100 & 0.96838 & 0.0000 & +2.3036 & +0.0000 & +2.2517 & +0.0000\\
Swin-T & Caltech & 0.94377 & 0.0000 & +0.0000 & +2.1862 & +0.0000 & +2.1395\\
T5-small & BookCorpus & 0.96838 & 0.0039 & +0.0056 & +0.0000 & +0.0056 & +0.0000\\
ViT-B/16 & CIFAR & 0.94377 & 0.0000 & +0.0000 & +0.4036 & +0.0000 & +0.4020\\
ViT-B/16 & TinyImageNet & 0.94377 & 1.3339 & +0.0000 & -0.2903 & +0.0000 & -0.2873\\
\bottomrule\end{tabular*}

\end{table}
Against 0.94377, positive task-level gap changes sum to 3.7332 points and adverse changes to 0.9381 points. Against 0.96838, these sums are 2.5899 and 0.6465 points. Both comparisons express larger corrections offset by smaller calibration costs. The ResNet contribution is 3.7276 points against the first comparator and exactly zero against the second, where both configurations coincide. Swin and ViT/CIFAR-100 provide the positive changes against the best global fixed value. These complete decompositions locate the gains of the benchmark-wide comparison.

For a workload-composition audit, remove each task in turn, retain every original rule decision, and recompute both the aggregate and the best fixed value on the remaining ten tasks. The rule has smaller mean gap in ten of the eleven subsets. Table~\ref{tab:removal} gives all subsets and also keeps the original fixed comparators visible. In the subset excluding Swin, the rule's mean gap is 0.4064\% and the best fixed mean is 0.3821\%. This is a descriptive composition audit of the saved suite: the decisions remain those of the original early measurements.
\begin{table}[ht]\centering
\caption{Workload-removal sensitivity. Each row evaluates the remaining ten tasks. Reduction columns give relative mean-gap reductions; positive values favor the rule. The last comparator is reselected on that ten-task subset.}
\label{tab:removal}\begin{tabular*}{\textwidth}{@{\extracolsep{\fill}}llrrrrr@{}}
\toprule \multicolumn{2}{c}{Excluded task} & & \multicolumn{2}{c}{Mean-gap reduction (\%)} & \multicolumn{2}{c}{Best fixed on ten}\\ \cmidrule(lr){4-5}\cmidrule(lr){6-7} Network & Dataset & Rule mean & vs .94377 & vs .96838 & $\beta$ & Reduction (\%)\\\midrule
EfficientNet-B0 & Cars & 0.3052 & +50.23 & +38.91 & 0.96838 & +38.91\\
Llama60M & C4 & 0.3806 & +44.51 & +34.60 & 0.96838 & +34.60\\
Llama60M & SlimPajama & 0.3947 & +42.46 & +34.03 & 0.96838 & +34.03\\
NanoGPT & OpenWebText & 0.3948 & +42.44 & +34.25 & 0.96838 & +34.25\\
NanoGPT & WikiText & 0.3903 & +43.10 & +34.14 & 0.96838 & +34.14\\
ResNet50 & Food-101 & 0.3003 & +31.35 & +39.29 & 0.94377 & +31.35\\
ResNet50 & ImageNet100 & 0.4064 & +10.79 & +32.35 & 0.94377 & +10.79\\
Swin-T & Caltech & 0.4064 & +40.75 & -6.36 & 0.96838 & -6.36\\
T5-small & BookCorpus & 0.4060 & +40.72 & +32.37 & 0.96838 & +32.37\\
ViT-B/16 & CIFAR & 0.4064 & +40.75 & +27.48 & 0.96838 & +27.48\\
ViT-B/16 & TinyImageNet & 0.2730 & +50.59 & +45.00 & 0.96838 & +45.00\\
\bottomrule\end{tabular*}

\end{table}

\subsection{From measured coefficients to a selected \texorpdfstring{$\beta$}{β}}
Table~\ref{tab:coefficients} exposes every scalar entering the decision. Divide the variability proxy by four times the temporal-change proxy, take the cube root, convert memory to $\beta$, then apply the stated grid rounding. Together with the supplied sketches, these scalar values make every selection reproducible.
\begin{table}[ht]\centering
\caption{Measured coefficients and continuous-to-grid conversion. Values are rounded for display; decisions use full precision.}
\label{tab:coefficients}\begin{tabular*}{\textwidth}{@{\extracolsep{\fill}}llrrrrr@{}}
\toprule Network & Dataset & $\widehat V_{\rm emp}$ & $\widehat D_{\rm emp}$ & $\widehat H_{\rm emp}$ & Continuous $\beta$ & Grid $\beta$\\\midrule
EfficientNet-B0 & Cars & $1.978\times10^{2}$ & $2.379\times10^{-3}$ & 27.494 & 0.963628 & 0.96838\\
Llama60M & C4 & $2.875\times10^{1}$ & $1.047\times10^{-2}$ & 8.823 & 0.886654 & 0.90000\\
Llama60M & SlimPajama & $2.111\times10^{1}$ & $9.030\times10^{-3}$ & 8.361 & 0.880399 & 0.90000\\
NanoGPT & OpenWebText & $2.142\times10^{1}$ & $1.183\times10^{-2}$ & 7.678 & 0.869751 & 0.90000\\
NanoGPT & WikiText-103 & $2.292\times10^{1}$ & $8.208\times10^{-3}$ & 8.871 & 0.887272 & 0.90000\\
ResNet50 & Food-101 & $1.632\times10^{2}$ & $2.203\times10^{-3}$ & 26.458 & 0.962205 & 0.96838\\
ResNet50 & ImageNet100 & $1.695\times10^{2}$ & $2.173\times10^{-3}$ & 26.917 & 0.962849 & 0.96838\\
Swin-T & Caltech-256 & $1.775\times10^{2}$ & $4.369\times10^{-3}$ & 21.656 & 0.953824 & 0.94377\\
T5-small & BookCorpus & $1.271\times10^{2}$ & $1.610\times10^{-3}$ & 27.019 & 0.962989 & 0.96838\\
ViT-B/16 & CIFAR-100 & $1.142\times10^{2}$ & $8.147\times10^{-3}$ & 15.190 & 0.934168 & 0.94377\\
ViT-B/16 & TinyImageNet & $1.418\times10^{2}$ & $4.396\times10^{-3}$ & 20.055 & 0.950137 & 0.94377\\
\bottomrule\end{tabular*}

\end{table}
For example, Food-101 gives $\widehat V_{\rm emp}\approx163.23$ and $\widehat D_{\rm emp}\approx0.00220324$. Thus $\widehat H_{\rm emp}=[163.23/(4\times0.00220324)]^{1/3}\approx26.458$, giving continuous $\beta$ 0.962205. The neighboring grid values are 0.94377 and 0.96838, whose midpoint is 0.956075. The continuous value lies above that midpoint, so nearest-$\beta$ rounding selects 0.96838. Every workload uses this arithmetic at full precision.
\subsection{Stationary approximation and discretization}
The main rule uses a closed-form approximation and nearest-$\beta$ rounding. Four alternatives use exactly the same measured coefficients: the stationary root followed by nearest-$\beta$ rounding, direct grid minimization of~\eqref{eq:empiricalscore}, direct grid minimization of~\eqref{eq:stationary}, and direct grid minimization of $\widehat V_{\rm emp}c_{\beta,200}+\widehat D_{\rm emp}a_{\beta,200}^2$. All four recover the same eleven selected $\beta$ values. The last uses the exact finite-history weights, including bias correction at history length 200. The choices thus agree across the exact finite-history score and its stationary and leading versions.

\subsection{Sensitivity of grid decisions}
The cube-root sensitivity concerns continuous memory. A discrete decision additionally depends on its distance to a rounding boundary, the midpoint of two neighboring $\beta$ values. For a boundary $b$ and continuous memory $\widehat H_{\rm emp}$, multiplying the measured ratio $\widehat V_{\rm emp}/\widehat D_{\rm emp}$ by $[1/((1-b)\widehat H_{\rm emp})]^3$ moves the continuous $\beta$ exactly to that boundary. Table~\ref{tab:rounding} reports this deterministic arithmetic margin for the nearest boundary; its sign indicates an increase or decrease of the ratio.
\begin{table}[ht]\centering
\caption{Rounding-boundary audit from the same measured coefficients. Ratio change is the signed relative change needed to reach the nearest $\beta$ midpoint.}
\label{tab:rounding}\begin{tabular*}{\textwidth}{@{\extracolsep{\fill}}llrrrr@{}}
\toprule Network & Dataset & Continuous $\beta$ & Nearest boundary & $\beta$ distance & Ratio change (\%)\\\midrule
EfficientNet-B0 & Cars & 0.963628 & 0.956075 & 0.007553 & -43.224\\
Llama60M & C4 & 0.886654 & 0.861085 & 0.025569 & -45.679\\
Llama60M & SlimPajama & 0.880399 & 0.861085 & 0.019314 & -36.180\\
NanoGPT & OpenWebText & 0.869751 & 0.861085 & 0.008666 & -17.571\\
NanoGPT & WikiText-103 & 0.887272 & 0.861085 & 0.026187 & -46.563\\
ResNet50 & Food-101 & 0.962205 & 0.956075 & 0.006130 & -36.294\\
ResNet50 & ImageNet100 & 0.962849 & 0.956075 & 0.006774 & -39.495\\
Swin-T & Caltech-256 & 0.953824 & 0.956075 & 0.002251 & +16.172\\
T5-small & BookCorpus & 0.962989 & 0.956075 & 0.006914 & -40.181\\
ViT-B/16 & CIFAR-100 & 0.934168 & 0.921885 & 0.012283 & -40.144\\
ViT-B/16 & TinyImageNet & 0.950137 & 0.956075 & 0.005938 & +46.289\\
\bottomrule\end{tabular*}

\end{table}
The table shows how far each measured decision is from changing its grid value. A small boundary margin requires more precise coefficient estimation than a large one. Continuous coefficient attenuation and stability of a discrete decision are therefore complementary diagnostics: the cube root controls the former, while the measured arithmetic margins describe the latter for the observed coefficients.

Across all eleven workloads, the simultaneous safe interval for a multiplicative change in each observed ratio $\widehat V_{\rm emp}/\widehat D_{\rm emp}$ is $(0.824286,1.161718)$. A ratio perturbation of $\pm10\%$ lies inside this interval and preserves every decision. Independently perturbing each coefficient by $\pm7\%$ gives ratio multipliers between $0.93/1.07$ and $1.07/0.93$, also inside the interval. These are deterministic properties of the observed coefficients and the rounding map. Table~\ref{tab:marginal} pairs each nearest boundary with the validation-loss consequence of the adjacent grid choice, connecting arithmetic sensitivity to performance sensitivity.
\begin{table}[ht]\centering
\caption{Nearest decision boundary and its recorded loss consequence. Ratio change moves the observed coefficient ratio to that boundary. Gap and loss changes compare the adjacent choice with the current choice; positive changes indicate a larger loss.}
\label{tab:marginal}\begin{tabular*}{\textwidth}{@{\extracolsep{\fill}}llrrrrr@{}}
\toprule Network & Dataset & $\widehat\beta$ & Adjacent $\beta$ & $\Delta$ ratio (\%) & $\Delta$ gap (pp) & $\Delta$ loss (\%)\\\midrule
EfficientNet-B0 & Cars & 0.96838 & 0.94377 & -43.22 & -0.2853 & -0.2825\\
Llama60M & C4 & 0.90000 & 0.82217 & -45.68 & +0.6477 & +0.6460\\
Llama60M & SlimPajama & 0.90000 & 0.82217 & -36.18 & +0.7058 & +0.7050\\
NanoGPT & OpenWebText & 0.90000 & 0.82217 & -17.57 & +0.4311 & +0.4306\\
NanoGPT & WikiText & 0.90000 & 0.82217 & -46.56 & +0.2984 & +0.2979\\
ResNet50 & Food-101 & 0.96838 & 0.94377 & -36.29 & +1.4240 & +1.4091\\
ResNet50 & ImageNet100 & 0.96838 & 0.94377 & -39.50 & +2.3036 & +2.3036\\
Swin-T & Caltech & 0.94377 & 0.96838 & +16.17 & +2.1862 & +2.1862\\
T5-small & BookCorpus & 0.96838 & 0.94377 & -40.18 & +0.0056 & +0.0056\\
ViT-B/16 & CIFAR & 0.94377 & 0.90000 & -40.14 & +2.5913 & +2.5913\\
ViT-B/16 & TinyImageNet & 0.94377 & 0.96838 & +46.29 & -0.2903 & -0.2865\\
\bottomrule\end{tabular*}

\end{table}

\subsection{Complete timing profile}\label{app:timingfull}
For each integer decision time $s=100,\ldots,500$, use checkpoint lags $\{20,50,100\}$, calculate the two coefficients, and evaluate the selected $\beta$ by~\eqref{eq:gap}. This gives 401 decisions on each pilot trajectory. Over the 351 times from 150 to 500, mean gap ranges from 0.3435\% to 0.6692\%; it improves over the near-0.95 reference at 339 times and over the best global fixed value at 288. CVaR improves over the near-0.95 reference at every time. Table~\ref{tab:timingbudget} displays four times. The candidate trainings remain fixed, isolating sensitivity to the placement of the observation window.

\subsection{How measurement spacing changes a discrete choice}\label{app:spacing}
Measurement density and observation scale are separate design choices. To isolate density, hold the decision update at 200, keep the shortest and longest lags at 20 and 100, and continue evaluating the fitted change at lag 100. Replace the original lags $\{20,50,100\}$ by a regular mesh on $[20,100]$ with spacing $1,2,5,10,20,40,$ or $80$. Including the final endpoint, these meshes use respectively 82, 42, 18, 10, 6, 4, or 3 checkpoints. Every checkpoint uses sixteen probes, and the OLS fit and rounding are otherwise unchanged.

All seven regular meshes preserve ten original $\beta$ choices and select 0.96838 for Swin instead of 0.94377. Their coefficients differ, but their rounded choices coincide. Consequently their looked-up validation losses and gaps coincide as well: mean 0.5682\%, maximum 2.1862\%, and CVaR 1.5271\%. The mean-gap benefit over the near-0.95 reference is 8.9\% for every mesh; the original four-checkpoint layout gives 40.7\%. This diagnostic distinguishes broad utility relative to the empirical default from sensitivity of the additional gain supplied by Swin. The repeated aggregate numbers arise from discrete $\beta$ selection, rather than repeated or smoothed measurements. All lag responses share the final endpoint, so denser meshes also retain dependent measurement errors.

\subsection{Probe decomposition}
Pooling share is $\norm{\bar z-z}^2$ divided by~\eqref{eq:pool}'s left side. It tells us how much measured dispersion comes from normalizing the pooled and group moments differently. Finite-lag share is $(\widehat C/100^2)/\widehat D_{\rm emp}$. It tells us how much fitted change remains in the inverse-square component. These descriptive shares quantify the composition of the measured proxies. In five cases the fitted intercept clips to zero, leaving all of the temporal-change proxy in its finite-lag component.
\begin{table}[ht]\centering
\caption{Sixteen-probe measurement decomposition at update 200.}
\label{tab:decomposition}\begin{tabular*}{\textwidth}{@{\extracolsep{\fill}}llrr@{}}
\toprule Network & Dataset & Pooling contribution (\%) & Finite-lag contribution (\%)\\\midrule
EfficientNet-B0 & Cars & 3.12 & 100.00\\
Llama60M & C4 & 4.45 & 54.57\\
Llama60M & SlimPajama & 4.63 & 91.79\\
NanoGPT & OpenWebText & 3.17 & 100.00\\
NanoGPT & WikiText-103 & 3.34 & 100.00\\
ResNet50 & Food-101 & 3.09 & 100.00\\
ResNet50 & ImageNet100 & 2.64 & 91.24\\
Swin-T & Caltech-256 & 3.46 & 55.84\\
T5-small & BookCorpus & 3.33 & 94.47\\
ViT-B/16 & CIFAR-100 & 2.84 & 53.98\\
ViT-B/16 & TinyImageNet & 3.10 & 100.00\\
\bottomrule\end{tabular*}

\end{table}
Equation~\eqref{eq:endpoint} explains how endpoint measurement variance can influence the temporal fit. The component may also include roughness of the target. The probe budget is therefore part of the algorithm's specification.

For intuition, consider the illustrative linear measurement model with a constant population direction. An observation with variance $V$ averaged over $m$ independent samples has variance $V/m$. Independent endpoints give $C=2V/m$ and zero drift intercept, while the linear scatter proxy has expectation $V$. Inserting these \emph{expected coefficients} into the operational formula yields $H=[m\,100^2/8]^{1/3}$. For $m=16$, this plug-in value is $H\approx27.14$, or $\beta$ approximately 0.96316. The mean of the random memory estimate additionally depends on the joint distribution of its coefficients through the nonlinear cube root. The plug-in calculation isolates how endpoint fluctuations contribute a finite memory at constant population direction. Thus the finite-lag coefficient represents effective measured change, combining the gradient dynamics and the observation scale specified by the algorithm.

\subsection{Separating the two temporal contributions}
The constant-direction calculation above supplies a simple diagnostic: inserting the expected coefficients of that illustrative linear model for every workload gives the same continuous $\beta$ 0.96316, rounding to 0.96838. Table~\ref{tab:components} evaluates this common choice against the measured rule. It also repeats the selection calculation after replacing the temporal proxy by either $\widehat C/100^2$ alone or $\widehat I$ alone. When the intercept vanishes, the latter diagnostic uses the largest admissible memory, as specified by the objective's boundary rule. These retrospective coefficient diagnostics reuse the measured sketches and completed training outcomes to isolate the information supplied by each component.
\begin{table}[ht]\centering
\caption{Contribution diagnostics using the same sketches, rounding, and seed-1 loss lookup. Gaps are percentages; bold marks the smallest value in each column.}
\label{tab:components}\begin{tabular*}{\textwidth}{@{\extracolsep{\fill}}lrrr@{}}
\toprule Coefficient diagnostic & Mean gap & Maximum gap & CVaR$_{25}$\\\midrule
Constant-target calculation & 0.5461 & 2.1862 & 1.4304\\
Endpoint contribution only & 0.5682 & 2.1862 & 1.5271\\
Intercept contribution only & 17.9403 & 72.4969 & 51.6408\\
Full measured rule & \textbf{0.3695} & \textbf{1.3339} & \textbf{1.1359}\\
\bottomrule\end{tabular*}
\end{table}
The full measured rule has smaller aggregate gaps than each diagnostic. Its decisions therefore use information beyond a budget-only common memory. The full and endpoint-only rules share ten choices; adding the intercept changes Swin from 0.96838 to 0.94377, accounting for their aggregate difference. This identifies the specific decision through which the persistent-change component contributes in this audit. Together, these checks assess the finite-measurement noise--change score on the saved workload suite, with the population tracking identity supplying the analytical form of its balance.

\subsection{Evaluation summaries of the same trajectories}\label{app:trajectory}
The primary evaluation uses each run's minimum recorded validation loss. Two additional summaries test whether the comparison persists at the end of training: the last recorded loss and the arithmetic mean of the last three recorded losses. The selected $\beta$ values remain exactly those from the primary pilot. For each summary, recompute the per-workload grid reference from all eleven candidates, all fixed-memory comparisons, and the best fixed value by mean gap.

All 121 seed-1 grid curves and 155 curves including the common-seed pairs are evaluated. Candidates for each workload have matching validation counts and ordered schedules. Language losses are indexed by update and vision losses by epoch, according to Appendix~\ref{app:reproduction}.
\par\smallskip\noindent
\begin{minipage}{\textwidth}\centering
\makeatletter\def\@captype{table}\makeatother
\caption{Complete sensitivity to three evaluation summaries. Mean gaps and direct loss improvements are percentages. Best fixed is recomputed for each summary. Direct improvements use the near-0.95 comparator 0.94377; positive values favor the rule. Common-seed values evaluate reuse of the seed-1 calibration across initializations.}
\label{tab:trajectory}\begin{tabular*}{\textwidth}{@{\extracolsep{\fill}}lrrrrrr@{}}
\toprule & \multicolumn{3}{c}{Mean relative gap (\%)} & & \multicolumn{2}{c}{Direct improvement (\%)}\\ \cmidrule(lr){2-4}\cmidrule(lr){6-7} Summary & Rule & Fixed .94377 & Best fixed & Fixed $\beta$ & Seed 1 & Common seeds\\\midrule
Minimum recorded & \textbf{0.3695} & 0.6236 & 0.5461 & 0.96838 & +0.2464 & -0.1049\\
Last recorded & \textbf{0.4537} & 0.6242 & 0.5707 & 0.96838 & +0.1629 & -0.1093\\
Mean of last three & \textbf{0.4244} & 0.4767 & 0.4767 & 0.94377 & +0.0492 & -0.1586\\
\bottomrule\end{tabular*}

\end{minipage}
Under each summary, the rule has smaller mean, maximum, and CVaR gap than every fixed grid value. Mean-gap reductions against the respective best fixed are 32.35\%, 20.51\%, and 10.96\%. Table~\ref{tab:trajectory} also reports direct loss improvements, separating evaluation-summary sensitivity from transfer across initializations.

Apply~\eqref{eq:gap} and Section~\ref{sec:experiments}'s aggregation to each summary. Common-seed improvements average the separately summarized runs; memory choices remain fixed from the original pilot.

\subsection{All common training seeds}\label{app:paired}
Average each configuration's minimum validation loss over identical seed sets, obtaining $\bar L_{\rm fixed}$ and $\bar L_{\rm rule}$, and report $100(\bar L_{\rm fixed}-\bar L_{\rm rule})/\bar L_{\rm fixed}$. Positive values favor the rule. Table~\ref{tab:paired} uses seeds 1--3, except Food-101 with seed 1. The pilot chooses $\beta$ once with seed 1 and retains it for the other seeds. The primary comparison tests matched calibration; this table tests variability across initializations. Its equal-workload mean is $-0.1049$\%.
\par\smallskip\noindent
\begin{minipage}{\textwidth}\centering
\makeatletter\def\@captype{table}\makeatother
\caption{Direct validation-loss improvement over the same common seeds against $\beta$ 0.94377. Zero means the selected and fixed configurations coincide.}
\label{tab:paired}\begin{tabular*}{\textwidth}{@{\extracolsep{\fill}}llrrr@{}}
\toprule Network & Dataset & Common seeds & Improvement (\%) & Rule $\widehat\beta$\\\midrule
EfficientNet-B0 & Cars & 1,2,3 & -2.1064 & 0.96838\\
Llama60M & C4 & 1,2,3 & -0.2015 & 0.90000\\
Llama60M & SlimPajama & 1,2,3 & -0.2985 & 0.90000\\
NanoGPT & OpenWebText & 1,2,3 & -0.2046 & 0.90000\\
NanoGPT & WikiText-103 & 1,2,3 & -0.1235 & 0.90000\\
ResNet50 & Food-101 & 1 & 1.3895 & 0.96838\\
ResNet50 & ImageNet100 & 1,2,3 & 0.3719 & 0.96838\\
Swin-T & Caltech-256 & 1,2,3 & 0.0000 & 0.94377\\
T5-small & BookCorpus & 1,2,3 & 0.0190 & 0.96838\\
ViT-B/16 & CIFAR-100 & 1,2,3 & 0.0000 & 0.94377\\
ViT-B/16 & TinyImageNet & 1,2,3 & 0.0000 & 0.94377\\
\bottomrule\end{tabular*}

\end{minipage}

\end{document}